\documentclass{article}

\usepackage{amsthm, amsmath,amssymb}
\usepackage{amstext}
\usepackage{fullpage}
\usepackage{dsfont}
\usepackage{xspace}
\usepackage{url}

\usepackage{algorithm}
\usepackage{algorithmic}

\newtheorem{theorem}{Theorem}
\newtheorem{lemma}{Lemma}
\newtheorem{corollary}{Corollary}

\renewcommand{\epsilon}{\varepsilon}
\newcommand{\R}{\mathds{R}}
\newcommand{\N}{\mathds{N}}

\newcommand{\ones}[1]{\OM(#1)}
\newcommand{\floor}[1]{\left\lfloor#1\right\rfloor}
\newcommand{\ceil}[1]{\left\lceil#1\right\rceil}
\newcommand{\E}[1]{\text{E}\left(#1\right)}

\newcommand{\Prob}[1]{\text{P}\left(#1\right)}
\newcommand{\LO}{\textup{LO}\xspace}
\newcommand{\OM}{\textup{OneMax}\xspace}

\newcommand{\polylog}{\mathrm{polylog}}

\newcommand{\algo}{\ensuremath{\mathcal{A}}\xspace}

\newcommand{\EA}{\text{(1+1)~EA}\xspace}
\newcommand{\EAmu}{\text{(1+1)~EA}\ensuremath{_\mu}\xspace}

\newcommand{\muea}{($\mu$+1)~EA\xspace}

\newcommand{\ie}{i.\,e.\xspace}

\newcommand{\Tea}[1]{\widetilde{E}^{\textrm{\OM}}_{\EAmu}(#1)}
\newcommand{\Tmuea}[1]{E^f_{\algo}(#1)}

\allowdisplaybreaks[3]

\begin{document}

\title{A New Method for Lower Bounds on the Running Time of~Evolutionary Algorithms\thanks{A preliminary version with parts of the results has been presented at a conference~\cite{Sudholt2010a}. The results therein were limited to mutation rate~$1/n$.}}

\author{Dirk Sudholt}

\maketitle

\begin{abstract}
We present a new method for proving lower bounds on the expected running time of evolutionary algorithms. It is based on fitness-level partitions and an additional condition on transition probabilities between fitness levels.
The method is versatile, intuitive, elegant, and very powerful. It yields exact or near-exact lower bounds for \LO, \OM{}, long $k$-paths, and all functions with a unique optimum.
Most lower bounds are very general: they hold for all evolutionary algorithms that only use bit-flip mutation as variation operator---i.\,e. for all selection operators and population models. The lower bounds are stated with their dependence on the mutation rate.

These results have very strong implications. They allow to determine the optimal mutation-based algorithm for \LO{} and \OM{}, i.\,e., which algorithm minimizes the expected number of fitness evaluations. This includes the choice of the optimal mutation rate.
\end{abstract}

\section{Introduction}

Evolutionary algorithms (EAs) and other randomized search heuristics have been successfully applied to countless difficult practical problems. One important reason for their popularity and their success is that they can be applied to a broad range of problems. They are usually easy to implement and they typically produce reasonable results in short time, with little effort.

However, getting the best possible results requires much greater effort.
When aiming for maximum efficiency, one has to think carefully about what search algorithm to use, how to make design choices, and how to tune the parameters of the algorithm.
In the search for the best strategy, researchers and practitioners alike are faced with a range of fundamental questions:
\begin{itemize}
\item How effective is search algorithm $\mathcal{A}$ on problem/problem class~$P$?
\item What is the best parameter setting for~$\mathcal{A}$ on~$P$?
\item Is search algorithm $\mathcal{B}$ faster than $\mathcal{A}$ on $P$?
\item What is the best search algorithm for~$P$?
\end{itemize}
Finding answers to these questions is now more pressing then ever. The field of evolutionary computation has grown immensely in the last decades and it has led to the development of countless variants of search algorithms, with new bio-inspired optimization paradigms emerging every year. This can prove to be a burden as practitioners are faced with an overwhelming variety of search algorithms.

Running time analysis has emerged as an important and very active area in evolutionary computation. The goal is to formally analyze the random or expected time until an evolutionary algorithm has found a satisfactory solution for a given problem.
By assessing how the expected running time grows with the problem dimension, we can gain valuable insights into their scalability. These insights apply to arbitrary, not too small problem dimensions---even to very large dimensions that are beyond the capabilities of today's hardware.

It also yields a solid foundation for the comparison of different EAs or different heuristic paradigms. This includes the question in how far design choices affect performance such as the choice of representations, operators, and parameters. In some cases running time analyses allow to draw conclusions about optimal parameter settings. Some of the above questions can be answered. Last but not least, theoretical analyses lead to insight into the working principles of EAs and to a better understanding of their behavior.

Running time analyses have been performed for classes of pseudo-Boolean functions such as unimodal functions~\cite{Droste2002}, linear functions~\cite{Droste2002,He2004,Jagerskupper2011,Doerr2010,Witt2011a}, functions with plateaus~\cite{Jansen2002a}, monotone polynomials~\cite{Wegener2005c}, and monotone functions~\cite{Doerr2010c}.
The same approach has been used for the analysis of problems from combinatorial optimization, see the survey by Oliveto, He, and Yao~\cite{Oliveto2007} or the recent text book by Neumann and Witt~\cite{BookNeuWit}. Also many other metaheuristics have been studied such as memetic algorithms~\cite{Sudholt2009,Sudholt2010,Sudholt2010b}, estimation-of-distribution algorithms~\cite{Droste2006a,Chen2010}, ant colony optimization~\cite{Gutjahr2008a,Neumann2009b,Neumann2009,Sudholt2011a}, particle swarm optimization~\cite{Sudholtsubmitteda,Witt2009}, and artificial immune systems~\cite{Zarges2008,Jansen2011}. A good summary of recent developments is given in the edited book by Auger and Doerr~\cite{Auger2011}.

However, running time analysis comes with several drawbacks.
In many cases running time analyses are very challenging. Search heuristics represent complex dynamic systems that are often hard to handle analytically. Hence, studies have often been limited to very specific settings. Comparisons between different search algorithms---or variants of the same algorithm---have often been performed on contrived artificial functions that were designed specifically to enable an analysis.

Furthermore, many analyses are restricted to a single, very specific algorithm such as the \EA with mutation probability~$1/n$. This helps to keep the analyses simple, but it also means that conclusions are limited to this particular algorithm.
Another shortcoming is that, when considering polynomial expected running times, often only upper bounds on the expected running time are shown. Upper bounds are more appealing than lower bounds as they show that a particular search algorithm is effective on a particular problem. Lower bounds are typically harder to prove and often more imprecise, compared to upper bounds. For example, for the function \OM an upper bound with the exact leading constant (i.\,e., the constant factor preceding the fastest-growing term) is known from the 1990s (see Rudolph~\cite[page 95]{Rudolph1997a}). But a matching lower bound with the same leading constant was only proved recently, in 2010, by Doerr, Fouz, and Witt~\cite{Doerr2010a}.

When only upper bounds are available it is hard to make comparisons between different algorithms. Even when an upper bound for search algorithm $\mathcal{A}$ is much lower than an upper bound for $\mathcal{B}$, we cannot conclude with rigor that $\mathcal{A}$ is more efficient than $\mathcal{B}$. It could be that the analysis for $\mathcal{A}$ is more precise than that for $\mathcal{B}$, but in fact $\mathcal{B}$ is more efficient than~$\mathcal{A}$. One has to take care not to draw wrong conclusions when interpreting running time bounds.
Only if we have a lower bound for $\mathcal{B}$ that is larger than the upper bound for~$\mathcal{A}$ we can say with certainty that $\mathcal{A}$ is more effective than~$\mathcal{B}$. This stresses the importance of lower bounds, and that of having precise running time bounds.

Many researchers have tried to develop methods for proving lower bounds. Drift analysis has emerged as one powerful tool~\cite{Oliveto2011,He2004,Doerr2010b,Lehre2010a,Doerr2011c}. However, it is not always easy to apply.
We present a new method for proving lower bounds on the running time of stochastic search algorithms (see Section~\ref{sec:method}). It follows the idea of fitness-based partitions or fitness levels, a well-known tool for proving upper running time bounds.
The idea is to partition the search space into a sequence of sets called fitness levels. These sets have to be traversed in order to find a global optimum. Lower bounds can be derived if we have upper bounds on the probability of reaching a better fitness level and additional information about the transition probabilities between fitness levels.

The method is illustrated with applications to well-studied test problems.
The function $\OM(x) := \sum_{i=1}^n x_i$ counts the number of ones in the bit string. The optimum is the all-ones bit string. Assessing the performance of a search algorithm on \OM equals the question how effective the algorithm is at hill climbing---and at finding a particular target point if best possible hints are given.
The function LeadingOnes, shortly $\LO(x) := \sum_{i=1}^n \prod_{j=1}^i x_i$, is another popular test function that counts the number of leading ones in the bit string. All bits have to be optimized sequentially. This gives an example of a unimodal function that is more difficult than \OM. It also resembles worst-case inputs for shortest path problems~\cite{Sudholt2011a}.
Long $k$-paths~\cite{Horn1994,Rudolph1997,Droste2002,Sudholt2009} represent even more difficult unimodal functions where EAs typically climb up a path. As the path can have exponential length and shortcuts are unlikely, this a very challenging problem. For details we refer to Section~\ref{sec:long-k-paths}.

The example applications show that the new method is applicable to a wide range of problems and to a very broad class of evolutionary algorithms. We introduce the term \emph{mutation-based EAs} for a class of EAs that first generate initial search points uniformly at random, and afterwards only use common bit-flip mutation operators for variation. This class contains all common EAs that do not use crossover, e.\,g., all $(\mu+\lambda)$~EAs, all $(\mu,\lambda)$~EAs as well as parallel variants such as island models. Basically, the class contains all EAs regardless of the selection operators and population models (see Section~\ref{sec:preliminaries}).

The resulting lower bounds apply to \emph{all} mutation-based EAs. They are not only tight in an asymptotic sense.
They contain best possible leading constants when compared to upper bounds for the best EAs in this class, up to lower-order terms.
The bounds also show how the expected running time depends on the mutation rate. This highlights the impact of this parameter on performance and it allows for conclusions on the optimal mutation rate. Along the way, we also present a refinement of the fitness-level method for proving upper bounds in Section~\ref{sec:refined-upper-bound}.

The results allow to make conclusions about optimal EAs, where optimality is regarded as minimizing the expected number of function evaluations. A summary of the results derived from applying the new method is as follows.
\EAmu denotes a variant of the \EA initialized with a best out of $\mu$ individuals generated uniformly at random.
\begin{itemize}
\item For \LO we get a lower bound for all mutation-based EAs, see Section~\ref{sec:LO}. This bound equals a refined upper bound for the \EAmu. For all $\mu$ we get an exact formula for the expected running time of the \EAmu, including the \EA. Together with the independent work by B{\"o}ttcher, Doerr, and Neumann~\cite{Boettcher2010}, this is the first time that an exact formula for an expected running time of an EA can be determined.
    Following~\cite{Boettcher2010}, the optimal mutation rate can be computed as $p\approx 1.59/n$. The optimal mutation-based EA turns out to be the \EAmu for some value $\mu > 1$.
\item For \OM we also get a lower bound for all mutation-based EAs, see Section~\ref{sec:onemax}. For all reasonable mutation rates the lower bound matches an upper bound for the \EA using the same mutation rate, up to terms of smaller order.
    The optimal mutation rate turns out to be $p=1/n$ (see also Witt~\cite{Witt2011a}). The optimal mutation-based EA is again the \EAmu for a proper $\mu > 1$.
\item The above lower bound on \OM generalizes to the very large class of functions that have a unique optimum, see Section~\ref{sec:unique}. This is based on the structural insight that for all mutation-based EAs finding a single target point for any problem is never easier than optimizing \OM.
\item For long $k$-paths we get upper and lower bounds that match up to smaller order terms, for all reasonable mutation rates, when considering the \EA starting on the first point on the path, see Section~\ref{sec:long-k-paths}. Like for \OM, $p=1/n$ is the optimal mutation rate.
\end{itemize}
In addition to these remarkably powerful results, the method is easy to describe and it has a simple, direct proof. As such, it is well suited for teaching purposes and it shows that precise lower bounds can be obtained without using drift analysis.

\subsection{Previous and Related Work}

There is a long history of results on pseudo-Boolean optimization. We review results on lower bounds and also describe work that preceded, relates to, or has followed from this work~\cite{Sudholt2010a}.

Already Droste, Jansen, and Wegener~\cite{Droste2002} presented a lower bound of $\Omega(n \log n)$ for the \EA on every $n$-bit pseudo-Boolean function with unique global optimum.
The constant factor preceding the $n \log n$-term is $1/2 \cdot (1-e^{-1/2}) \approx 0.196$.
Wegener~\cite{Wegener2002} mentions a lower bound $(1-\varepsilon) \cdot n \ln n -cn$ where $\varepsilon > 0$ is an arbitrarily small constant and the constant $c > 0$ depends on~$\varepsilon$.
Doerr, Fouz, and Witt~\cite{Doerr2010a} presented a lower bound $(1 - o(1)) en \ln n$ for the \EA on \OM.
The last result was extended by Doerr, Johannsen, and Winzen~\cite{Doerr2010}. They proved that the same bound holds for the \EA on every function with a unique global optimum.

Later on, Doerr, Fouz, and Witt~\cite{Doerr2011c} were inspired by the lower bound $en \ln n - 2\log \log n - 16n$ for mutation-based EAs with $p=1/n$ on \OM{} in the preliminary work~\cite{Sudholt2010a}. Their goal was to remove the $2\log \log n$-term in order to arrive at an even more precise bound for \EA. They managed to get a lower bound of $en \ln n - O(n)$ and along the way they introduced two new techniques to the analysis of randomized search heuristics: lower bounds with variable drift and probability-generating functions.

Witt~\cite{Witt2011a} followed up on this work~\cite{Sudholt2010a} and presented lower bounds for the class of mutation-based EAs on linear functions. He proved that the mutation rate $p=1/n$ is an optimal choice for the \EA on linear functions. He also generalized a structural result from~\cite{Sudholt2010a} in the following sense. The original statement is that the expected optimization time of any mutation-based EA with mutation probability~$1/n$ on any function with unique global optimum is at least as large as the expected optimization time of the \EA with mutation probability~$1/n$ on \OM{}. Witt generalized this towards arbitrary mutation probabilities~$0 < p \le 1/2$ and stochastic dominance. We will discuss and apply this result in Section~\ref{sec:unique}.

The LeadingOnes function has been equally popular. Droste, Jansen, and Wegener~\cite{Droste2002} showed that the running time of the \EA on \LO{} is at least $c_1 n^2$ with probability $1-2^{-\Omega(n)}$, for some constant $c_1 > 0$.
B{\"o}ttcher, Doerr, and Neumann presented an exact formula for the expected running time of the \EA on \LO at the same conference~\cite{Boettcher2010}. While the preliminary version of this work~\cite{Sudholt2010a} considered mutation rates of $p=1/n$ only, the authors considered general mutation rates~$p$. Their results were limited to the \EA as opposed to all mutation-based EAs. They showed that the optimal fixed mutation rate for \LO{} is not $p=1/n$, but a slightly higher value of $p \approx 1.59/n$. In addition, they presented a simple adaptive scheme for choosing the mutation rate and showed that this leads to even smaller numbers of function evaluations.

These findings show that the often recommended choice~$p=1/n$ is not always optimal. Another reason why the choice of the mutation probability is far from settled is that even on a seemingly easy class of functions a constant factor in the mutation rate can change a polynomial expected running time into an exponential one~\cite{Doerr2010c}.

Black-box complexity of search algorithms as introduced by Droste, Jansen, and Wegener~\cite{Droste2006} is another method for proving lower bounds. These bounds hold for all algorithms in a black-box setting where only the class of functions to be optimized is known, but the precise instance is hidden from the algorithm. Their results imply that \emph{every} black-box algorithm needs at least $\Omega(n/\!\log n)$ function evaluations to optimize \OM{} and \LO{} (or, to be more precise, straightforward generalizations to function classes).
Recently Lehre and Witt~\cite{Lehre2010} presented a more restricted black-box model. If only unary operators are used (that is, operators taking a single search point as input, such as mutation) and all operators are unbiased with respect to bit values and bit positions, every black-box algorithm needs $\Omega(n \log n)$ function evaluations for every function with a unique global optimum. The constant factor hidden in the $\Omega$ is not specified; it is known to be at most~$1$.
This line of research has been extended subsequently to more general conditions for unbiasedness~\cite{Rowe2011}, higher-arity operators~\cite{Doerr2011a} and more restricted black-box models~\cite{Doerr2011b}.

Investigating conditions for the optimality of search algorithms, Borisovsky and Eremeev~\cite{Borisovsky2008} introduced the concept of dominance for the performance comparison of evolutionary algorithms. For sorting problems and the function \OM{} they give sufficient conditions on when the \EA is faster than evolutionary algorithms with other reproduction operators.

Recently, drift analysis has received a lot of attention~\cite{Oliveto2011,He2004,Doerr2010b,Lehre2010a,Doerr2011c}. Assume a non-negative potential function such that the optimum is reached only if the potential is 0. If the expected decrease (``drift'') of the potential in one generation is bounded from below, an upper bound on the expected optimization time follows. Conversely, an upper bound on the drift implies lower bounds on the expected optimization time.
If there is a drift pointing away from the optimum on a part of the potential's domain then exponential lower bounds can be shown~\cite{Lehre2010a,Oliveto2011}.

\section{Preliminaries}
\label{sec:preliminaries}

The presentation in this work is for maximization problems, but it can be easily adapted for minimization. For the usage of asymptotic notation we refer to text books such as Cormen, Leiserson, Rivest, and Stein~\cite{Cormen2001}.

\subsection{Mutation-Based Evolutionary Algorithms}

The technique for proving lower bounds will be applied to a very general class of evolutionary algorithms. It contains all EAs that generate $\mu \in \N$ individuals uniformly at random and afterwards only use standard mutations to generate offspring (see Algorithm~\ref{alg:construct}).

Mutation is done by flipping each bit independently with some given mutation probability~$0 < p \le 1/2$. The most extreme value $p=1/2$ corresponds to choosing an offspring uniformly at random, i.\,e., random search. We do not consider mutation rates $p > 1/2$ as this choice would favor offspring far away from the parent, thus contradicting the purpose of mutation.

\begin{algorithm}[h]
    \caption{Scheme of a mutation-based EA}
    \algsetup{indent=1.5em}
    \begin{algorithmic}[1]
        \STATE create $\mu$ individuals $x_1, \dots, x_\mu \in \{0, 1\}^n$ uniformly at random.
        \STATE let $t := \mu$.
        \LOOP
            \STATE select a parent $x \in \{x_1, \dots, x_t\}$ according to $t$ and $f(x_1), \dots, f(x_t)$.
            \STATE create $x_{t+1}$ by copying $x$ and flipping each bit independently with probability~$p$.
            \STATE let $t := t +1$.
        \ENDLOOP
    \end{algorithmic}
    \label{alg:construct}
\end{algorithm}

The optimization time is given by the time index~$t$ that counts the number of function evaluations. It is defined as the time index~$t$ when a global optimum is found first. In a more general sense, we can also regard the (expected) hitting time of a set of desirable search points. For some lower bounds and for small values of~$\mu$, we pessimistically disregard the effort for creating the $\mu$ search points.

The parent selection mechanism is very general as any mechanism based on the time index~$t$ and fitness values of previous search points may be used.
Any mechanism for managing a population fits in this framework. This includes parent populations and offspring populations with arbitrary selection strategies and even parallel evolutionary algorithms with spatial structures and migration such as the island model~\cite{Lassig2011}.

\label{def:EA}
The \EA is a well-known special case with population size~$\mu = 1$. It maintains a single individual $x$ and in every iteration it creates $x'$ by mutating~$x$ and replacing~$x$ by~$x'$ if $f(x') \ge f(x)$. We denote by \EAmu a generalization of the \EA that is initialized with a best individual out of $\mu$ individuals which are generated uniformly at random.

Before introducing the new lower-bound method we elaborate on the range of sensible values for the mutation rate~$p$. The expected number of flipping bits equals $pn$. This is 1 for the standard choice~$p=1/n$. If $p \ll 1/n$ then the expected number of flipping bits is close to~0. The expected time until mutation creates any offspring that is different from its parent is then at most~$1/(pn)$ as $pn$ is an upper bound on the probability that any bit flips. This means that $p$ must be at least an inverse polynomial to allow for polynomial expected running times (unless the initialization finds a global optimum with high probability).

If the problem only contains a single optimum that has to be hit, $p$ cannot be too large. If $p \le 1/2$ then the best probability for hitting the optimum from a non-optimal parent is obtained when the parent has Hamming distance~1 to the optimum. Then the probability is~$p(1-p)^{n-1} \le (1-p)^n \le e^{-pn}$ and the expected waiting time until this happens is~$e^{pn}$.
We summarize these findings in the following theorem, showing that unreasonable parameter settings lead to unreasonable running times. Note that the optimum is not found during initialization with population size~$\mu$ with probability at least $1-\mu \cdot 2^{-n}$.
\begin{theorem}
\label{the:pointless-mutation-rates}
Let $f$ be a function with a unique global optimum. The expected optimization time of every mutation-based EA on~$f$ with mutation probability $0 < p \le 1/2$ is at least $(1-\mu \cdot 2^{-n}) \cdot 1/(pn)$ and at least $(1-\mu \cdot 2^{-n}) \cdot e^{pn}$.

In particular, for every~$\mu$ the expected optimization time is superpolynomial if $p \le n^{-\omega(1)}$ or $p = \omega(\log n)/n$ and exponential (i.\,e. $2^{n^{\varepsilon}}$ for some constant $\varepsilon > 0$) if $p \le 2^{-n^{\Omega(1)}}$ or $p = n^{\Omega(1)-1}$.
\end{theorem}
The result can be extended towards functions with multiple global optima, but the above result suffices for our purposes.

\subsection{The Fitness-Level Method for Proving Upper Bounds}

We review the \emph{fitness-level method}, also known as the \emph{method of $f$-based partitions}~\cite{Wegener2002}.
It yields upper bounds for EAs whose best fitness value in the population never decreases. We call these algorithms \emph{elitist EAs}.

The idea is as follows. We partition the search space into sets that are strictly ordered with respect to the fitness of the contained individuals. Every search point in a higher fitness-level set has a strictly higher fitness than any search point in a lower fitness-level set.
We say that an elitist algorithm is on a particular level if the best search point created to far is in the respective fitness-level set. Due to the elitism, the algorithm can only increase its current fitness level.
If we have a lower bound on the probability of increasing the current level, the reciprocal is an upper bound on the expected time until a particular fitness level is left. As each level is left for good, the sum of all these times---starting from the initial level---yields an upper bound on the expected optimization time.

\begin{theorem}[Fitness-level method for proving upper bounds]
\label{the:fitness-levels}
For two sets $A, B \subseteq \{0, 1\}^n$ and fitness function~$f$ let $A <_f B$ if $f(a) < f(b)$ for all $a \in A$ and all $b \in B$.
Consider a partition of the search space into non-empty sets $A_1, \dots, A_m$ such that
$
A_1 <_f A_2 <_f \dots <_f A_m
$
and $A_m$ only contains global optima.
For a mutation-based EA $\algo$ we say that $\algo$ is in $A_i$ or on level $i$ if the best individual created so far is in~$A_i$. Consider some elitist EA $\algo$ and let $s_i$ be a lower bound on the probability of creating a new offspring in $A_{i+1} \cup \dots \cup A_m$, provided $\algo$ is in~$A_i$.
Then the expected optimization time of $\algo$ on $f$ (without the cost of initialization) is bounded by
\begin{equation}
\label{eq:upper-bound-with-fitness-levels}
\sum_{i=1}^{m-1} \Prob{\text{$\algo$ starts in $A_i$}} \sum_{j=i}^{m-1} \frac{1}{s_i}
 \;\le\; \sum_{i=1}^{m-1} \frac{1}{s_i}.
\end{equation}
\end{theorem}
The second bound results from pessimistically assuming that the algorithm is always initialized in $A_1$.

Let us illustrate the method with two examples for the \EA with mutation probability~$p=1/n$.
We define the \emph{canonical partition} as the partition in which $A_i$ contains exactly all search points with fitness~$i$.
For \LO{} the method applied to the canonical partition yields an upper bound of $\sum_{i=0}^{n-1} en = en^2$ since the probability of finding an improvement is lower bounded by the probability of flipping the first bit with value~$0$. This probability is at least $1/n \cdot (1-1/n)^{n-1} \ge 1/(en)$. For \OM{} we get an upper bound of $\sum_{i=0}^{n-1} en/(n-i) = en \sum_{i=1}^{n} 1/i \le en\ln n + O(n)$ for the \EA since on level~$i$ there are $n-i$ 1-bit mutations that flip a 0-bit to~1 and hence improve the fitness.

In order to make an effort towards a unified theory of search heuristics, we also present the following extension. After finding an improvement, stochastic search algorithms often need some time to adapt their underlying probabilistic models. For instance, the algorithm \muea investigated by Witt~\cite{Witt2006} needs some time until the population contains ``enough'' individuals on a new fitness level, so that an improvement can be found with a good probability.
The ant colony optimization algorithms investigated in Gutjahr and Sebastiani~\cite{Gutjahr2008a} as well as in Neumann, Sudholt, and Witt~\cite{Neumann2009} need some time to adapt their pheromones towards a new best solution. A similar argument holds for velocities in a binary particle swarm optimization algorithm investigated by Sudholt and Witt~\cite{Sudholt2008c}.

In all these studies, it is pessimistically disregarded that an improvement might be found while waiting for the algorithm to adapt.
Fix a notion of adaptation and let $T_i$ be the (random) time until an algorithm has adapted, after a new fitness level $i$ has been found.
Redefining $p_i$ to the worst-case probability of finding an improvement in one iteration after adaptation, the expected optimization time can be bounded by
\[
\sum_{i=1}^{m-1} \Prob{\text{$\algo$ starts in $A_i$}} \sum_{j=i}^{m-1} \left(\E{T_j} + \frac{1}{s_i}\right)
 \;\le\; \sum_{i=1}^m \left(\E{T_i} + \frac{1}{s_i}\right).
\]

In addition, Lehre~\cite{Lehre2011} recently presented an extension towards non-elitist populations, with applications to comma strategies and various selection operators. Roughly speaking, he proves that if
\begin{itemize}
\item the probability of generating an offspring on a worse fitness level is not too large,
\item selection has a strong enough tendency to pick high-fitness individuals, and
\item the population is large enough
\end{itemize}
then an upper bound similar to the one in Theorem~\ref{the:fitness-levels} applies. The running time bound is asymptotic, not revealing a precise constant factor, though. But his work shows that the method is applicable in a much more general context.

\section{Lower Bounds with Fitness Levels}
\label{sec:method}

We now show that fitness-level arguments can also be applied to show tight lower bounds on the running time. Researchers have attempted to make this step earlier.
The best lower bounds with fitness-level arguments known so far were presented by Wegener in~\cite{Wegener2002}.
\begin{lemma}[Wegener~\cite{Wegener2002}]
\label{lem:crude-lower-bound-method}
Let $A_1 <_f \dots <_f A_m$ be a fitness-level partition for some fitness function~$f$. Let $u_i$ be an upper bound on the probability of an EA $\algo$ creating a new offspring in $A_{i+1} \cup \dots \cup A_m$, provided $\algo$ is in $A_i$ (where ``$\algo$ is in $A_i$'' is defined as in Theorem~\ref{the:fitness-levels}). Then the expected optimization time of $\algo$ on~$f$ is at least
\[
\sum_{i=1}^{m-1} \Prob{\text{$\algo$ starts in $A_i$}} \frac{1}{u_i}.
\]
\end{lemma}
The resulting lower bounds are very weak since we only look at the time it takes to leave the initial fitness level and then pessimistically assume that the optimum is found.

For instance, for the \EA with mutation probability~$p=1/n$ on \OM{} Lemma~\ref{lem:crude-lower-bound-method} yields the lower bound
\[
\sum_{i=0}^{n-1} \binom{n}{i} \cdot 2^{-n} \cdot \frac{n-i}{n} \approx 2
\]
as the initialization is very likely to create a search point with around $n/2$ 1-bits.
For the \EA on \LO{} we get the lower bound
\[
\sum_{i=0}^{n-1} 2^{-i-1} \cdot \frac{1}{1/n} = (1-2^{-n}) \cdot n,
\]
which is again very crude; the real expected running time is of order $\Theta(n^2)$.

Much better lower bounds can be achieved by making an additional assumption about the transition probabilities between fitness levels. The idea is as follows. If we know that a search algorithm typically does not skip too many fitness levels, it is likely that many fitness levels need to be traversed. This yields a lower bound that is proportional to the upper bound from Theorem~\ref{the:fitness-levels}.

In the following theorem $\gamma_{i, j}$ can be regarded as the conditional probability of jumping from level~$i$ to level~$j$, given that the algorithm leaves level~$i$.
\begin{theorem}
\label{the:lower-bound-method}
Consider an algorithm~\algo and a partition of the search space into non-empty sets
$A_1, \dots, A_m$.
For a mutation-based EA $\algo$ we again say that $\algo$ is in $A_i$ or on level $i$ if the best individual created so far is in~$A_i$. Let the probability of \algo traversing from level $i$ to level $j$ in one step be at most $u_i \cdot \gamma_{i,j}$ and $\sum_{j=i+1}^{m} \gamma_{i, j} = 1$.
Assume that for all $j > i$ and some $0 \le \chi \le 1$ it holds
\begin{equation}
\label{eq:gamma-condition}
\gamma_{i, j} \ge \chi \sum_{k=j}^{m} \gamma_{i, k}.
\end{equation}
Then the expected hitting time of $A_m$ is at least
\begin{align}
&
\sum_{i=1}^{m-1} \Prob{\text{$\algo$ starts in $A_i$}} \cdot \left(\frac{1}{u_i} + \chi \sum_{j=i+1}^{m-1} \frac{1}{u_j}\right)\label{eq:complex-lower-bound}\\
\ge\;& \sum_{i=1}^{m-1} \Prob{\text{$\algo$ starts in $A_i$}} \cdot \chi \sum_{j=i}^{m-1} \frac{1}{u_j}.\label{eq:simple-lower-bound}
\end{align}
\end{theorem}

The variable~$\chi$ was coined \emph{viscosity} by Jon Rowe~\cite{RowePersonal}. Similar to the viscosity of a liquid, it resembles the viscosity of the fitness-level partition on a scale between 0 and 1. A low viscosity means that we can have situations where a search algorithm  skips many fitness levels and only few levels are actually encountered. A high viscosity means that a search algorithm typically encounters many fitness levels as large jumps to higher fitness levels are unlikely.

For $\chi > 0$ the reciprocal, $1/\chi$, is an upper bound on the expected number of fitness levels gained during an improvement. To see this, note that condition~\eqref{eq:gamma-condition} implies
$
\sum_{k=j}^m \gamma_{i, k} \le (1-\chi) \cdot \sum_{k=j-1}^m \gamma_{i, k}
$
for all $j > i$. This implies $\sum_{k=j}^m \gamma_{i, k} \le (1-\chi)^{j-i-1} \sum_{k=i+1}^{m} \gamma_{i, k} = (1-\chi)^{j-i-1}$.
Using $\E{X} = \sum_{x=0}^{\infty} \Prob{X \ge x}$ if $X$ takes only non-negative integer values, the expected progress in terms of fitness levels, assuming current level~$i$, is at most
\[
\sum_{j=i+1}^{m} \sum_{k=j}^m \gamma_{i, k} \le \sum_{j=i+1}^m (1-\chi)^{j-i-1} = \sum_{j=0}^{m-i-1} (1-\chi)^j = \frac{1-(1-\chi)^{m-i}}{\chi}.
\]
This yields $1/\chi$ as an upper bound that is independent from the current level.

In a case of extreme viscosity, i.\,e., $\chi =1$ condition~\eqref{eq:gamma-condition} can only hold if $\gamma_{i, i+1} = 1$ and $\gamma_{i, k} = 0$ for all $1 \le i \le m-1$ and all $2 \le k \le m-i$. This means that the algorithm deterministically reaches the next fitness level when an improvement is made. It passes through all fitness levels between the initial one and the optimal one. These are the strongest possible conditions on the transition probabilities.

Contrarily, if $\chi = 0$ we have no viscosity at all. Condition~\eqref{eq:gamma-condition} is trivially satisfied for all choices of the $\gamma$-variables. This is the weakest possible setting and it leaves open the possibility that the optimum is reached by a direct jump, when the current fitness level is left. In fact, the resulting bound~\eqref{eq:complex-lower-bound} equals the one from Lemma~\ref{lem:crude-lower-bound-method}.

The most interesting settings are those where the viscosity is between 0 and 1.
For instance, if $\chi = 1/2$ then condition~\eqref{eq:gamma-condition} is roughly equivalent to the $\gamma$-variables decreasing exponentially with base~2: $\gamma_{i, i+k} \le 2^{-k}$. Larger viscosities require a steeper decay, while smaller viscosities allow for a less steep decay.
For selected fitness levels on \OM{} it turns out that the transition probabilities decay rapidly, allowing to choose $\chi$ as high as $1-o(1)$. This means that only a vanishing fraction of fitness levels is skipped---in expectation---and it leads to a very tight lower bound.

Before we get to the proof of Theorem~\ref{the:lower-bound-method}, we state the following conclusions about how tight the upper and lower bounds with fitness levels can be.
\begin{corollary}
Let $A_1, \dots, A_m, \chi, s_i, u_i$, and $\gamma_{i, j}$ for $1 \le i, j \le m$ be defined as in Theorems~\ref{the:fitness-levels} and \ref{the:lower-bound-method}. Let all conditions in these theorems hold.
\begin{enumerate}
\item If $s_i = u_i$ for all~$i$ then the lower bound~\eqref{eq:simple-lower-bound} matches the upper bound~\eqref{eq:upper-bound-with-fitness-levels} up to a factor of $\chi$.
\item If $\chi > 0$ is a constant and there is a constant~$c \ge 1$ such that $u_i \le c \cdot u_i$ for $1 \le i \le m-1$ then \eqref{eq:simple-lower-bound} and \eqref{eq:upper-bound-with-fitness-levels} are asymptotically equal.
\item If $\chi = 1-o(1)$ and $u_i \le (1+o(1)) \cdot s_i$ for $1 \le i \le m-1$ then \eqref{eq:simple-lower-bound} and \eqref{eq:upper-bound-with-fitness-levels} are equal up to lower-order terms.
\end{enumerate}
\end{corollary}
A fitness-level partition that obeys the second case was called \emph{(asymptotically) tight $f$-based partition} in~\cite{Lassig2011}.

We proceed by proving Theorem~\ref{the:lower-bound-method}. Afterwards, we give advice on how to apply it, including example applications in the following sections.
\begin{proof}[Proof of Theorem~\ref{the:lower-bound-method}]
The second bound immediately follows from the first one since $0 \le \chi \le 1$.
Let $E_i$ be the minimum expected remaining optimization time, where the minimum is taken for all possible histories $x_1, \dots, x_t$ of previous search points with $x_1, \dots, x_t \in A_1 \cup \dots \cup A_i$. By definition $E_1 \ge E_2 \ge \dots \ge E_m = 0$ as the conditions on the histories are subsequently relaxed.
By the law of total expectation the unconditional expected optimization time is at least
$
\sum_{i=1}^{m-1} \Prob{\text{$\algo$ starts in $A_i$}} \cdot E_i
$,
hence we only need to bound $E_i$.

After one step, for each $i < k$ the algorithm is in $E_k$ with probability at most $u_i \gamma_{i, k}$ and it remains in $i$ with probability $1-\sum_{k=i+1}^{m-1}u_i \gamma_{i, k} = 1-u_i$. This establishes
the recurrence
\[
E_i \ge 1 + \sum_{k=i+1}^{m-1} u_i \gamma_{i, k} \cdot E_k + (1-u_i) \cdot E_i.
\]
Subtracting $(1-u_i) E_i$ on both sides and dividing by $u_i$ yields
\[
E_i \ge \frac{1}{u_i} + \sum_{j=i+1}^{m-1} \gamma_{i, j} \cdot E_j.
\]
Assume for an induction that for all $k > i$ it holds
$
E_k \ge \frac{1}{u_k} + \chi \sum_{j=k+1}^{m-1} \frac{1}{u_j}
$.
Then we get
\begin{equation}
\label{eq:formula-Ei}
E_i \ge \frac{1}{u_i} + \sum_{k=i+1}^{m-1} \gamma_{i, k} \cdot \left(\frac{1}{u_k} + \chi \sum_{j=k+1}^{m-1} \frac{1}{u_j}\right).
\end{equation}
Note that
\begin{equation}
\label{eq:rearranging-sums}
\sum_{k=i+1}^{m-1} \gamma_{i, k} \cdot \chi \sum_{j=k+1}^{m-1} \frac{1}{u_j} = \sum_{j=i+1}^{m-1} \frac{1}{u_j} \cdot \chi \sum_{k=i+1}^{j-1} \gamma_{i, k}
\end{equation}
since on the left-hand side every term $1/u_j$ appears for all summands $k= i+2, \dots, j-1$ in the outer sum, each summand weighted by $\gamma_{i, k} \chi$.
Together, we get
\begin{align*}
E_i \stackrel{\eqref{eq:formula-Ei}}{\ge}\;&
\frac{1}{u_i} + \sum_{k=i+1}^{m-1} \gamma_{i, k} \cdot \left(\frac{1}{u_k} + \chi \sum_{j=k+1}^{m-1} \frac{1}{u_j}\right)\\
 \stackrel{\eqref{eq:rearranging-sums}}{=}\;& \frac{1}{u_i} + \sum_{j=i+1}^{m-1} \frac{1}{u_j} \left(\gamma_{i, j} + \chi \sum_{k=i+1}^{j-1} \gamma_{i, k}\right)\\
\stackrel{\eqref{eq:gamma-condition}}{\ge}\;& \frac{1}{u_i} + \sum_{j=i+1}^{m-1} \frac{1}{u_j} \left(\chi \sum_{k=j}^{m} \gamma_{i, k} + \chi \sum_{k=i+1}^{j-1} \gamma_{i, k}\right)\\
 =\;& \frac{1}{u_i} + \chi \sum_{j=i+1}^{m-1} \frac{1}{u_j}.
\end{align*}
\end{proof}
One crucial asset of the theorem is that in order to apply it, we do not need to know the transition probabilities exactly. It suffices to state upper bounds on the transition probabilities.
More precisely, we require that $u_i \cdot \gamma_{i,j}$ is an upper bound on the probability of jumping from level~$i$ to level~$j$.
We have the freedom to choose $u_i$ and $\gamma_{i, j}$ as long as the $\gamma$-variables sum up to~1 and they fulfil~\eqref{eq:gamma-condition}.

In cases where the transition probabilities are not known precisely or where it is not possible or feasible to derive an analytical expression, we can use different $\gamma$-variables as substitutes. Note that the condition on $u_i \cdot \gamma_{i, j}$ upper bounding the real transition probability is easier to fulfil if $u_i$ is large. So, we can choose $u_i$ as large as necessary in order to prove the conditions of Theorem~\ref{the:lower-bound-method}. The price for choosing a large $u_i$ is that the resulting lower bound becomes smaller as the $u_i$'s grow.

A similar observation holds for the choice of~$\chi$.
As remarked before, the higher the viscosity $\chi$, the stronger the conditions on the transition probabilities are.
The lower~$\chi$, the easier it is to establish condition~\eqref{eq:gamma-condition}, and the smaller the lower bound becomes.

The method is hence very versatile and flexible as we are free to choose $\chi$ and the $u$-, $\gamma$-variables such that all conditions hold. The upcoming example applications give advice as to how these values can be chosen.

Note that the theorem does not require the sets $A_i$ to form fitness levels: we do not assume $A_1 <_f \dots <_f A_m$. The conditions on the $\gamma$-variables indirectly imply that sets with small index are ``worse'' than sets with higher index.
Also note that Theorem~\ref{the:lower-bound-method} bounds the expected hitting time of set $A_m$. This includes the expected optimization time as special case in which $A_m$ contains exactly all global optima.
Alternatively, $A_m$ can contain other desirable solutions such as those with a certain minimum fitness, all local optima, all feasible solutions, etc.

\section{Refined Upper Bounds with Fitness Levels}
\label{sec:refined-upper-bound}

It has become clear that information about the transition probabilities is essential for proving meaningful lower bounds. This knowledge can also help to obtain refined upper bounds. The following result is very similar to Theorem~\ref{the:lower-bound-method}, with some inequalities reversed. Also the proof ideas are very similar to the ones in Theorem~\ref{the:lower-bound-method}.
In contrast to the lower bound, we need to add the condition $(1-\chi) s_{j} \le s_{j+1}$ for all $1 \le j \le m-2$, which states that the success probabilities must not be imbalanced.
\begin{theorem}
\label{the:refined-upper-bound}
Consider a partition of the search space into non-empty sets\linebreak
$A_1 <_f A_2 <_f \dots <_f A_m$ such that only $A_m$ contains global optima.
For an elitist mutation-based EA $\algo$ we again say that $\algo$ is in $A_i$ or on level $i$ if the best individual created so far is in~$A_i$. Let the probability of traversing from level $i$ to level $j$ in one step be at least $s_i \cdot \gamma_{i,j}$ and $\sum_{j=i+1}^{m} \gamma_{i, j} = 1$.
Assume that for all $j > i$ and some $0 < \chi \le 1$ it holds
\begin{equation}
\label{eq:gamma-condition-upper-bounds}
\gamma_{i, j} \le \chi \sum_{k=j}^{m} \gamma_{i, k}.
\end{equation}
Further assume $(1-\chi) s_{j} \le s_{j+1}$ for all $1 \le j \le m-2$.
Then the expected hitting time of $A_m$ is at most
\begin{align}
&
\sum_{i=1}^{m-1} \Prob{\text{$\algo$ starts in $A_i$}} \cdot \left(\frac{1}{s_i} + \chi \sum_{j=i+1}^{m-1} \frac{1}{s_j}\right)\label{eq:refined-upper-bound}.
\end{align}
\end{theorem}
For maximum viscosity, i.\,e., $\chi = 1$, the condition $(1-\chi) s_{j} \le s_{j+1}$ as well as condition~\eqref{eq:gamma-condition-upper-bounds} are always true. We then get the classical fitness-level method from Theorem~\ref{the:fitness-levels}. The refined upper bound method from Theorem~\ref{the:refined-upper-bound} is hence more general than the classical method from Theorem~\ref{the:fitness-levels}. Lower viscosities lead to better upper bounds. For instance, a constant viscosity between 0 and 1 typically reduces the upper bound by a constant, compared to Theorem~\ref{the:fitness-levels}. Unlike for lower bounds, a viscosity of $\chi = 0$ is impossible. Similar to the lower bound, we have that $\frac{1-(1-\chi)^{m-i}}{\chi}$ is now an upper bound for the expected number of gained fitness levels in an improvement from level~$i$.
\begin{proof}[Proof of Theorem~\ref{the:refined-upper-bound}]
Let $E_i$ be the worst-case expected remaining optimization time, given that the algorithm is in~$A_i$. The worst case is over all histories that contain at least one search point in~$A_i$. By the law of total expectation the unconditional expected optimization time is at most
$
\sum_{i=1}^{m-1} \Prob{\text{$\algo$ starts in $A_i$}} \cdot E_i
$,
hence we only need to bound $E_i$.

Assume for an induction that for all $i < k \le m-1$ it holds
\[
E_k \le \frac{1}{s_k} + \chi \sum_{j=k+1}^{m-1} \frac{1}{s_j} := b_k,
\]
$b_k$ denoting an upper bound for $E_k$.
The assumption holds trivially for $i=m-1$.

We now claim that $E_i \le b_i$. Note that the bounds are non-increasing: $b_{i+1} \ge b_{i+2} \ge \dots \ge b_{m-1}$. The reason is that for all $j > i$ we have
\[
b_{j} - b_{j+1} = \frac{1}{s_j} - \frac{1}{s_{j+1}} + \frac{\chi}{s_{j+1}} = \frac{1-\chi}{(1-\chi)s_{j}} - \frac{1-\chi}{s_{j+1}} \ge 0
\]
as $(1-\chi)s_j \le s_{j+1}$ by assumption.
Now, if $E_i \le b_{i+1}$ then also $E_i \le b_i$ and the claim follows. We therefore assume $E_i > b_{i+1}$ in the following, which implies $E_i > b_j$ for all $j > i$. Intuitively, this means that, when relying on~$E_i$ and the upper bounds~$b_{i+1}, \dots, b_{m-1}$, leaving $A_i$ towards any $A_j$, $j > i$, is always better than staying in~$A_i$. We are being pessimistic if we overestimate the probability of staying in~$A_i$.

This justifies the following recurrence. After one step the algorithm is in $A_k$ with probability at least $s_i \gamma_{i, k}$ and then the expected remaining optimization time is bounded by~$b_k$. The algorithm remains in $A_i$ with probability $1-\sum_{k=i+1}^{m-1}s_i \gamma_{i, k} = 1-s_i$ and then the remaining time is again bounded by~$E_i$. This gives
\[
E_i \le 1 + \sum_{k=i+1}^{m-1} s_i \gamma_{i, k} \cdot b_k + (1-s_i) \cdot E_i
\]
and rearranging yields
\[
E_i \le \frac{1}{s_i} + \sum_{j=i+1}^{m-1} \gamma_{i, j} \cdot b_j.
\]
Then we get
\begin{align*}
E_i \le\;& \frac{1}{s_i} + \sum_{k=i+1}^{m-1} \gamma_{i, k} \cdot \left(\frac{1}{s_k} + \chi \sum_{j=k+1}^{m-1} \frac{1}{s_j}\right)\\
 \stackrel{\eqref{eq:rearranging-sums}}{=}\;& \frac{1}{s_i} + \sum_{j=i+1}^{m-1} \frac{1}{s_j} \left(\gamma_{i, j} + \chi \sum_{k=i+1}^{j-1} \gamma_{i, k}\right)\\
\stackrel{\eqref{eq:gamma-condition-upper-bounds}}{\le}\;& \frac{1}{s_i} + \sum_{j=i+1}^{m-1} \frac{1}{s_j} \left(\chi \sum_{k=j}^m \gamma_{i, k} + \chi \sum_{k=i+1}^{j-1} \gamma_{i, k}\right)\\
 =\;& \frac{1}{s_i} + \chi \sum_{j=i+1}^{m-1} \frac{1}{s_j}.\qquad\qedhere
\end{align*}
\end{proof}

\section{An Exact Formula for LeadingOnes}
\label{sec:LO}

Our first application of the lower-bound method is for \LO{} as here the $\gamma$-values can be estimated in a very natural and precise way.
\begin{theorem}
\label{the:lower-LO}
Let $X_\mu$ be a random variable that describes the maximum \LO-value among $\mu$ individuals created independently and uniformly at random.
For every $n \ge 2$ the expected optimization time of every mutation-based EA on \LO{} using mutation probability $0 < p \le 1/2$ is at least
\begin{align}
& \sum_{i=0}^{n-1} \Prob{X_\mu = i} \cdot \frac{1}{p} \left(\left(1-p\right)^{-i} + \frac{1}{2} \sum_{j=i+1}^{n-1} \left(1-p\right)^{-j}\right)\label{eq:LO-general-lower}\\
=\;& \sum_{i=0}^{n-1} \Prob{X_\mu = i} \cdot \frac{1}{2p^2} \left(\left(1 - p\right)^{-n+1} - (1-2p) \left(1 - p\right)^{-i}\right)\label{eq:LO-general-lower-refined}\\
\ge\;&
\frac{1}{2p^2} \left(\left(1 - p\right)^{-n+1} - 1\right) - \frac{O(\log n)}{p},\label{eq:LO-precise-lower}
\end{align}
the last inequality holding for $p \ge n^{-\Omega(1)}$ and $p \le \ln \ln n \cdot 1/n$.
\end{theorem}
\begin{proof}
Consider the canonical partition and assume that the algorithm is on level~$i < n$. This implies that in the best individual created so far the first $i+1$ bits are predetermined. In addition, in all individuals created so far the bits at positions $i+2, \dots, n$ have not contributed to the fitness yet. These bits have been initialized uniformly at random and they have been subjected to random mutations. It is easy to see that this again results in uniform random bits. More precisely, the probability that a specific bit~$j$ with $j \ge i+2$ in a specific individual has a specific bit value $0$ or $1$ is exactly $1/2$ (see the proof of Theorem~17 in Droste, Jansen, and Wegener~\cite{Droste2002}).

Consider an individual $x$ that has been selected as parent among the created individuals. Let $\LO(x) = j \le i$. We bound the probability of creating an offspring with $k$ leading ones for some $i+1 \le k \le n$.
One necessary condition is that the first $j$ leading ones do not flip, which happens with probability $(1-p)^j$. The bit at position $j+1$ is 0, hence it must be flipped. All bits at positions $j+2, \dots, i+1$ must obtain the value 1 in the offspring. This probability is determined by the number of ones among these bits. But clearly $(1-p)^{i-j}$ is a lower bound on this probability since this reflects the best-case scenario that all these bits are 1 in the parent. (Since $p \le 1/2$ the probability of flipping a bit is not larger than the probability of not flipping it.)
The last necessary condition is to create exactly $k-1-i$ ones among at positions $i+2, \dots, n$.
By the preceding arguments on the ``randomness'' of these bits, the probability of creating exactly $k-1-i$ ones is $2^{-k+i} := \gamma_{i, k}$ if $k < n$ and $2^{-k+i+1} := \gamma_{i, k}$ if $k=n$.
Putting everything together, we have that
$p\left(1 - p\right)^{i} \cdot \gamma_{i, k}$
is an upper bound on the probability of jumping to level~$k$.

Checking the condition on the $\gamma$-values,
$\sum_{k=i+1}^{n} \gamma_{i, k} = \sum_{k=i+1}^{n-1} 2^{-k+i} + 2^{-n+i+1}
= 1$
and for all $i < j \le n$ condition~\eqref{eq:gamma-condition} holds with equality since
\[
\sum_{k=j}^{n} \gamma_{i, k} = \sum_{k=j}^{n-1} 2^{-k-i-1} + 2^{-n-i} = 2^{-j+i+1} = 2 \gamma_{i, j}.
\]
Setting $\chi = 1/2$, the preconditions for Theorem~\ref{the:lower-bound-method} are fulfilled. Using $u_i := p(1-p)^i$, this proves the bound
\[
\sum_{i=0}^{n-1} \Prob{X_\mu = i} \cdot  \left(\frac{1}{p} \cdot \left(1 - p\right)^{-i} + \frac{1}{2} \sum_{j=i+1}^{n-1} \frac{1}{p} \cdot \left(1 - p\right)^{-j}\right)
\]
and hence~\eqref{eq:LO-general-lower}.

For the second bound, observe that the bracketed term in~\eqref{eq:LO-general-lower} can be simplified as
\begin{align*}
 \left(1-p\right)^{-i} + \frac{1}{2} \sum_{j=i+1}^{n-1} \left(1-p\right)^{-j}
=\;& \frac{1}{2} \left(\left(1-p\right)^{-i} + \sum_{j=0}^{n-1} \left(1-p\right)^{-j} - \sum_{j=0}^{i-1} \left(1-p\right)^{-j}\right)\\
=\;& \frac{1}{2} \left(\left(1-p\right)^{-i} + \frac{1 - (1-p)^{-n}}{1-(1-p)^{-1}} - \frac{1 - (1-p)^{-i}}{1-(1-p)^{-1}}
\right)\\
=\;& \frac{1}{2} \left(\left(1-p\right)^{-i} + \frac{1-p}{p} \left(1 - p\right)^{-n} - \frac{1-p}{p} \left(1 - p\right)^{-i}\right)\\
=\;& \frac{1}{2p} \left(\left(1 - p\right)^{-n+1} - (1-2p) \left(1 - p\right)^{-i}\right).
\end{align*}
The third bound~\eqref{eq:LO-precise-lower} follows by simple calculations and the following case distinctions. Note that due to the asymptotic term $-O(\log n)$ we only need to prove the bound for large~$n$, i.\,e., for $n \ge n_0$ where we can fix $n_0 \in \N$.

Observe that the bound~\eqref{eq:LO-general-lower} is never larger than $\bar{\mu} := 1/(2p(1-p)^{n-1})$, even for the special case $X_\mu = 0$. If $\mu \ge \bar{\mu}$ then the probability that the optimum is not found during the first $\bar{\mu}$ individuals created during initialization is at most $\bar{\mu} \cdot 2^{-n} \le 1/n$ for $n$ large enough. This proves the claimed lower bound.

If $\mu \le \bar{\mu}$ then $\Prob{X_\mu > \log(\bar{\mu}/p)} \le \bar{\mu} \cdot 2^{-\log(\bar{\mu}/p)} = p$.
Pessimistically assuming that $X_\mu = \log(\bar{\mu}/p)$ in case $X_\mu \le \log(\bar{\mu}/p)$ and estimating the conditional expected optimization time by 0 in case $X_\mu > \log(\bar{\mu}/p)$ results in the following bound.
\begin{align*}
& \left(1 - p\right) \cdot \frac{1}{2p^2} \left(\left(1 - p\right)^{-n+1} - (1-2p) \left(1 - p\right)^{-\log(\bar{\mu}/p)}\right)\\
\ge\;& \frac{1}{2p^2} \left(\left(1 - p\right)^{-n+1} - \left(1 - p\right)^{-\log(\bar{\mu}/p)} - p(1-p)^{-n+1}\right).
\end{align*}
We use $(1-p)^{-n+1} \le (1-p)^{-n} \le e^{pn} \le e^{\ln \ln n} = \ln n$ to estimate the term $-p(1-p)^{-n+1}$. For the same reason $\log(\bar{\mu}/p) \le \log(1/(2p^2) \cdot \ln n) = O(\log n)$, recalling $p \ge n^{-\Omega(1)}$.
Assuming that $n$ is large enough to make $p \cdot \log(\bar{\mu}/p) \le \ln \ln n \cdot O((\log n)/n) \le 1/2$,
\[
(1-p)^{-\log(\bar{\mu}/p)} \le \frac{1}{1-p \cdot \log(\bar{\mu}/p)} = 1 + \frac{p \cdot \log(\bar{\mu}/p)}{1-p \cdot \log(\bar{\mu}/p)} \le 1 + O(p\log n).
\]
Together, we get a lower bound of
\begin{align*}
 \frac{1}{2p^2} \left(\left(1 - p\right)^{-n+1} - 1 - O(p \log n) - p \cdot \ln n\right)
=\;& \frac{1}{2p^2} \left(\left(1 - p\right)^{-n+1} - 1\right) - \frac{O(\log n)}{p}.
\end{align*}
\end{proof}
Note that a term $-O(\log n)/p$ is, in general, necessary since with, say, $\mu=n$ an EA will start with an average of $\Theta(\log n)$ leading ones in the best search point. As the \EA with mutation probability $\Theta(1/n)$ needs expected time $\Theta(n \log n)$ to collect $\Theta(\log n)$ leading ones, the \EAmu needs roughly $\Theta(n \log n) - n$ less generations than the \EA.

For the \EAmu $u_i \gamma_{i, j}$ is the exact probability of jumping from fitness level~$i$ to level~$j > i$.
Also recall that all conditions~\eqref{eq:gamma-condition} on the $\gamma_{i, j}$-values hold with equality.
Therefore, defining $s_i := u_i$ and using $\gamma_{i, j}$ and $\chi$ as in Theorem~\ref{the:lower-LO}, we get an upper bound for the \EAmu using Theorem~\ref{the:refined-upper-bound}. It is easy to see that $(1-\chi)s_i \le s_{i+1}$ for all $0 \le i \le n-2$ as $1/2 \cdot p(1-p)^i \le p(1-p)^{i+1}$ is equivalent to $1/2 \le 1-p$.
The resulting upper bound equals the lower bounds~\eqref{eq:LO-general-lower} and~\eqref{eq:LO-general-lower-refined} from Theorem~\ref{the:lower-LO}.

As the upper bound holds for the \EAmu but the lower bound holds for all mutation-based EAs, this proves that among all mutation-based EAs
the \EAmu is an optimal algorithm for the function \LO.
\begin{theorem}
\label{the:LO-upper-bound}
The term \eqref{eq:LO-general-lower}
describes the exact expected optimization time of the \EAmu with mutation probability $0 < p \le 1/2$ on \LO{}.
Among all mutation-based EAs with mutation probability~$0 \le p \le 1/2$, the \EAmu, for an appropriate choice of~$\mu$, minimizes the expected number of function evaluations.
\end{theorem}
For $\mu=1$ we get the following.
\begin{corollary}
\label{cor:(1+1)EA-LO}
The expected optimization time of the \EA with mutation probability~$0 < p \le 1/2$ on \LO{} is exactly
\[
\sum_{i=0}^{n-1} 2^{-i-1} \cdot \frac{1}{p} \left(\left(1-p\right)^{-i} +\frac{1}{2} \sum_{j=i+1}^{n-1} \left(1-p\right)^{-j}\right)
= \frac{1}{2p^2} \cdot \left((1-p)^{-n+1} - (1 - p)\right).
\]
\end{corollary}
The second bound follows from a simple but tedious calculation. It is omitted here.
For $p=1/n$ we get that the expected running time of the \EA is
\[
\frac{n^2}{2} \cdot \left(\left(1-\frac{1}{n}\right)^{-n+1} - 1 + \frac{1}{n}\right).
\]
The factor preceding $n^2$ converges to $(e-1)/2$ from below.
Note that we have reproduced one of the main results from B{\"o}ttcher, Doerr, and Neumann~\cite{Boettcher2010} for general mutation probabilities.
The latter authors derived the same formula and used it to compute the optimal mutation probability. They found that $p \approx 1.59/n$ is the optimal fixed mutation probability in that it minimizes the expected number of function evaluations. Our lower-bound method allows for the same conclusions to be drawn. Even stronger, while B{\"o}ttcher et al.~\cite{Boettcher2010} only consider the \EA, we can make the following statement for the broad class of mutation-based EAs.
\begin{theorem}
\label{the:optimal-algorithm-for-LO}
Among all mutation-based EAs the expected number of fitness evaluations on \LO{} is minimized by the \EAmu with mutation probability~$p=1.59/n$ and $1 < \mu = O(n \log n)$.
\end{theorem}
As shown by B{\"o}ttcher, Doerr, and Neumann~\cite{Boettcher2010}, the expected optimization time can be further decreased by allowing adaptive schemes for choosing the mutation probability. Theorems~\ref{the:LO-upper-bound} and~\ref{the:optimal-algorithm-for-LO} only apply to fixed mutation rates. This is not due to a limitation of the lower-bound method. The method is applicable to their adaptive algorithm as well. We refrain from going into detail as this would overlap to a large extend with results already published in~\cite{Boettcher2010}.

\section{A Lower Bound for OneMax}
\label{sec:onemax}

We turn to the function \OM instead. This function is the easiest function with a unique global optimum and it has been studied in the context of many search heuristics~\cite{Droste2002,Lassig2010a,Gutjahr2008a,Neumann2010a,Droste2006a,Sudholtsubmitteda}.
In this section we now derive a lower bound for the expected running time of all mutation-based EAs on \OM. This lower bound will be very close to a simple upper bound for the \EA. Using the fitness-level method for upper bounds, the expected running time of the \EA with mutation probability~$p$ can easily be bounded as follows.
\begin{theorem}
\label{the:upper-onemax}
Let $H(n)$ denote the $n$-th harmonic number. For any initial search point, the expected running time of the \EA with mutation probability~$p$, $0 < p < 1$, is bounded from above by
\[
\frac{H(n)}{p(1-p)^{n-1}} \le \frac{\ln n + 1}{p(1-p)^n}.
\]
\end{theorem}
\begin{proof}
Define the canonical fitness levels $A_i := \{x \mid \OM(x) = i\}$ for $0 \le i \le n$. The \EA increase the current fitness level~$i < n$ if only a single 0-bit flips and no 1-bit flips. This probability is at least
\[
s_i \ge \frac{n-i}{p(1-p)^{n-1}},
\]
resulting in the upper bound
\[
\sum_{i=0}^{n-1} \frac{1}{s_i} \le \frac{1}{p(1-p)^{n-1}} \cdot \sum_{i=0}^{n-1} \frac{1}{n-i} = \frac{H(n)}{p(1-p)^{n-1}}.
\]
The second bound follows from $H(n) \le (\ln n) + 1$.
\end{proof}
We remark that Witt~\cite[Theorem~4]{Witt2011a} recently presented a similar, but more complicated upper bound. It applies to all linear functions and also allows for tail bounds.

The main result in this section is the following lower bound.
\begin{theorem}
\label{the:lower-onemax}
The expected optimization time of every mutation-based EA using mutation probability~$p$ on \OM\ with $n \ge 2$ bits is at least
\[
\frac{\ln n - \ln \ln n - 3}{p(1-p)^n}
\]
if $2^{-n/3} \le p \le 1/n$ and at least
\[
\frac{\ln(1/(p^2n)) - \ln \ln n - 3}{p(1-p)^n}
\]
if $1/n \le p \le 1/(\sqrt{n} \log n)$.
\end{theorem}

For the default mutation probability~$p=1/n$, we get the following using the common estimation $1/n \cdot (1-1/n)^{n} \le 1/(en)$.
\begin{corollary}
\label{cor:lower-onemax}
The expected optimization time of every mutation-based EA using the default mutation probability~$p = 1/n$ on \OM\ is at least
\[
en \ln n - en \ln \ln n - 3en.
\]
\end{corollary}
Note that for mutation probabilities $p = \alpha/n$ for some polylogarithmic term $\alpha = \polylog(n)$ (defined as $O(\log^k n)$, $k > 0$ an arbitrary constant), the term $\ln(1/(p^2 n))$ in the second bound of Theorem~\ref{the:lower-onemax} simplifies to $\ln(n/\alpha^2) = \ln n - 2\ln(\alpha) = \ln n - o(\ln n)$. Hence, for mutation probabilities up to $\polylog(n)/n$, Theorem~\ref{the:lower-onemax} gives lower bounds that match the simple upper bound from Theorem~\ref{the:upper-onemax} up to lower-order terms.

An immediate conclusion from this result is that for the mentioned mutation probabilities the expected running time of the \EA is dominated by the term $\frac{\ln n}{p(1-p)^{n-1}}$. (Recall that for all mutation probabilities not covered by Theorem~\ref{the:lower-onemax} the expected running time is exponential by Theorem~\ref{the:pointless-mutation-rates}.) As $p(1-p)^{n-1}$ is maximized by the choice $p := 1/n$, the expected running time is minimized for this value, assuming that $n$ is large enough. This establishes $p=1/n$ as the optimal mutation rate for the \EA on \OM.

This finding has recently been derived independently by Witt~\cite{Witt2011a}. His result holds for all linear functions. The proof uses sophisticated drift analysis techniques. In this light it is surprising that the same statement (for \OM) can be derived by simple fitness level arguments. This further demonstrates the strength of the new lower bound method.

In order to show Theorem~\ref{the:lower-onemax}, we first show the following upper bounds on transition probabilities by mutation on \OM. The lemma may be of independent interest.
\begin{lemma}
\label{lem:p-i-to-i+k}
Let $p_{i, i+k}$ denote the probability that mutating a search point with $i$ 1-bits using mutation probability~$p$ results in an offspring with $i+k$ 1-bits. For every $k \in \N_0$ we have
\[
p_{i, i+k} \le p^k (1-p)^{n-k} \cdot \frac{(n-i)^k}{k!} \cdot \sum_{j=0}^n \left(\frac{i(n-i)p^2}{(1-p)^2}\right)^j \cdot \frac{1}{j!(j+1)!}.
\]
If, additionally, $\frac{i(n-i)p^2}{(1-p)^2} \le 1$ and $i \ge 2n/3$ then for every $0 \le i' \le i$
\[
p_{i', i+k} \le p^k (1-p)^{n-k} \cdot \frac{(n-i)^k}{k!} \cdot \left(1 + \frac{3}{5} \cdot \frac{i(n-i)p^2}{(1-p)^2}\right).
\]
\end{lemma}
The last statement means that, under the stated conditions, starting from a search point with a smaller number $i' < i$ of 1-bits does not give a better guarantee on the probability of jumping to level~$i+k$. This statement always holds for mutation probability~$p=1/n$, even without the mentioned conditions. However, for larger mutation probabilities this is non-trivial. There are examples where, under conditions different to the ones in Lemma~\ref{lem:p-i-to-i+k}, $p_{i', i+k} > p_{i, i+k}$ for $i' \le i$.
\begin{proof}[Proof of Lemma~\ref{lem:p-i-to-i+k}]
An offspring with $i+k$ 1-bits is created if and only if there is an integer $j \in \N_0$ such that $j$ 1-bits flip and $k+j$ 0-bits flip. Using $(k+j)! \ge k!(j+1)!$ for all $k \in \N$, $j \in \N_0$,
\begin{align*}
p_{i, i+k} =\;& \sum_{j=0}^n \binom{i}{j} \binom{n-i}{k+j} p^{k+2j} (1-p)^{n-k-2j}\\
=\;& p^k (1-p)^{n-k} \cdot \sum_{j=0}^n \binom{i}{j} \binom{n-i}{k+j} \left(\frac{p}{1-p}\right)^{2j}\\
\le\;& p^k (1-p)^{n-k} \cdot \sum_{j=0}^n \frac{i^j}{j!} \cdot \frac{(n-i)^{k+j}}{(k+j)!} \cdot \left(\frac{p}{1-p}\right)^{2j}\\
\le\;& p^k (1-p)^{n-k} \cdot \frac{(n-i)^k}{k!} \sum_{j=0}^n \frac{i^j}{j!} \cdot \frac{(n-i)^{j}}{(j+1)!} \cdot \left(\frac{p}{1-p}\right)^{2j}\\
=\;& p^k (1-p)^{n-k} \cdot \frac{(n-i)^k}{k!} \sum_{j=0}^n \left(\frac{i(n-i)p^2}{(1-p)^2}\right)^j \frac{1}{j!(j+1)!}.
\end{align*}
The second bound for $i' = i$ follows from
\[
\sum_{j=0}^n \left(\frac{i(n-i)p^2}{(1-p)^2}\right)^j \frac{1}{j!(j+1)!}
\le 1 + \sum_{j=1}^n \frac{i(n-i)p^2}{(1-p)^2} \cdot \frac{1}{j!(j+1)!}
\le 1 + \frac{i(n-i)p^2}{(1-p)^2} \cdot \frac{3}{5}.
\]
For $i' < i$ let $d := i - i'$. Note that $i \ge 2n/3$ implies
\[
\frac{p^2(n-i)^2}{(1-p)^2} \le \frac{1}{2} \cdot \frac{p^2 i(n-i)}{(1-p)^2} \le \frac{1}{2},
\]
hence $p(n-i)/(1-p) \le 1/\!\sqrt{2}$. Along with the first statement, we have
\begin{align*}
p_{i-d, i+k} \le\;& p^{d+k} (1-p)^{n-d-k} \cdot \frac{(n-i)^{d+k}}{(d+k)!} \sum_{j=0}^n \left(\frac{(i-d)(n-i+d)p^2}{(1-p)^2}\right)^j \frac{1}{j!(j+1)!}\\
\le\;& p^{k} (1-p)^{n-k} \cdot \frac{(n-i)^{k}}{k!} \cdot \frac{p^d}{(1-p)^d} \cdot \frac{(n-i)^d}{(d+1)!} \sum_{j=0}^n \left(\frac{(i-d)(n-i+d)p^2}{(1-p)^2}\right)^j \frac{1}{j!(j+1)!}\\
\le\;& p^{k} (1-p)^{n-k} \cdot \frac{(n-i)^{k}}{k!} \cdot \frac{2^{-d/2}}{(d+1)!} \sum_{j=0}^n \left(\frac{(i-d)(n-i+d)p^2}{(1-p)^2}\right)^j \frac{1}{j!(j+1)!}\\
\le\;& p^{k} (1-p)^{n-k} \cdot \frac{(n-i)^{k}}{k!} \cdot \frac{2^{-d/2}}{(d+1)!} \sum_{j=0}^n \left((d+1) \cdot \frac{i(n-i)p^2}{(1-p)^2}\right)^j \frac{1}{j!(j+1)!}\\
\le\;& p^{k} (1-p)^{n-k} \cdot \frac{(n-i)^{k}}{k!} \cdot \frac{2^{-d/2}}{(d+1)!} \sum_{j=0}^n \left(d+1\right)^j \frac{1}{j!(j+1)!}\\
\le\;& p^{k} (1-p)^{n-k} \cdot \frac{(n-i)^{k}}{k!}.
\end{align*}
\end{proof}

Now we proceed with the proof of the lower bound.
\begin{proof}[Proof of Theorem~\ref{the:lower-onemax}]
Assume that $n \ge 91$ as otherwise both bounds are negative and the claim is trivial.
If $\mu \ge \bar{\mu} := \frac{2\ln n}{p(1-p)^{n-1}}$ then the probability that the first $\bar{\mu}$ search points generated during initialization find the optimum is at most $\bar{\mu} \cdot 2^{-n} \ll 1/2$, which establishes the lower bound $\bar{\mu}/2 \ge \frac{\ln n}{p(1-p)^{n}}$ and proves both bounds.
In the following we assume $\mu \le \bar{\mu}$ and neglect the cost of initialization.

Let $\ell = \ceil{n-\min\{n/\!\log n, 1/(p^2 n\log n)\}}$. Consider the following partition $A_{\ell}, \dots, A_{n}$. Define $A_i = \{x \mid \ones{x} = i\}$ for $i > \ell$ and let $A_{\ell}$ contain all remaining search points.
With probability at least $1-\bar{\mu} \cdot \sum_{i=0}^{n-\ell} \binom{n}{i} 2^{-n} \ge 1-1/(\log n)$ for $n \ge 91$ the initial population only contains individuals on the first fitness level.

For $j > i$ let $p_{i, j}$ be the probability of the event that mutating an individual with $i$ ones results in an offspring that contains $j$ ones.
If $i \ge \ell$ then
\begin{equation}
\label{eq:bound-i(n-1)p^2}
i(n-i)p^2 \le n(n-\ell)p^2 \le \frac{1}{\log n} \le (1-p)^2.
\end{equation}
From Lemma~\ref{lem:p-i-to-i+k} we know that then for every $k \in \N_0$ and every $i' \le i$
\[
p_{i', i+k} \le p^k (1-p)^{n-k} \cdot \frac{(n-i)^k}{k!} \cdot \left(1 + \frac{3}{5} \cdot \frac{i(n-i)p^2}{(1-p)^2}\right).
\]
Without loss of generality, we can assume $i' := i$ in the following, i.\,e., that the algorithm always selects a best individual from the population as parent.
For $i \ge \ell$ define
\[
u_i' := p(1-p)^{n-1} \cdot (n-i) \cdot \left(1 + \frac{3}{5} \cdot \frac{i(n-i)p^2}{(1-p)^2}\right) \text{\quad and \quad} \gamma_{i,i+k}' := \left(\frac{p(n-i)}{1-p}\right)^{k-1}
\]
where the prime indicates that these will not be the final variables used in the application of Theorem~\ref{the:lower-bound-method}.
Observe that
\begin{align*}
u_i' \gamma_{i, i+k}' =\;& p(1-p)^{n-1} \cdot (n-i) \cdot \left(1 + \frac{3}{5} \cdot \frac{i(n-i)p^2}{(1-p)^2}\right) \cdot \left(\frac{p(n-i)}{1-p}\right)^{k-1}\\
=\;& p^k (1 - p)^{n-k} \cdot (n-i)^k \cdot \left(1 + \frac{3}{5} \cdot \frac{i(n-i)p^2}{(1-p)^2}\right)\\
\ge\;& p^k (1 - p)^{n-k} \cdot \frac{(n-i)^k}{k!} \cdot \left(1 + \frac{3}{5} \cdot \frac{i(n-i)p^2}{(1-p)^2}\right)
\;\ge\; p_{i, i+k}.
\end{align*}
Since Theorem~\ref{the:lower-onemax} requires the $\gamma_{i, j}$-variables to sum up to 1, we consider the following normalized variables:
$u_i := u_i' \cdot \sum_{j=i+1}^{n} \gamma_{i, j}'$ and $\gamma_{i, j} := \frac{\gamma_{i, j}'}{\sum_{j=i+1}^{n} \gamma_{i, j}'}$.
As $u_i \gamma_{i, j} = u_i' \gamma_{i, j}' \ge p_{i, j}$, the conditions on the transition probabilities are fulfilled.
The condition $\gamma_{i, j} \ge \chi \sum_{k=j}^n \gamma_{i, j}$ is equivalent to $\gamma_{i, j}' \ge \chi \sum_{k=j}^n \gamma_{i, j}'$.
Also note that
\begin{equation}
\label{eq:bound-p(n-i)}
p(n-i) \le p(n-\ell) \le \frac{1}{\log n} < 1-p,
\end{equation}
the second inequality following from $p(n-\ell) \le pn/\!\log n \le 1/\!\log n$ if $p \le 1/n$ and $p(n-\ell) \le p/(p^2 n \log n) = 1/(pn \log n) \le 1/\!\log n$ if $p \ge 1/n$.
Noting that $\frac{p(n-i)}{1-p} < 1$, we get
\begin{align*}
\sum_{k=j-i}^n \gamma_{i, i+k}'
= \sum_{k=j-i}^n \left(\frac{p(n-i)}{1-p}\right)^{k-1}
\le\;& \left(\frac{p(n-i)}{1-p}\right)^{j-i-1} \sum_{k=0}^{\infty} \left(\frac{p(n-i)}{1-p}\right)^{k}\\
=\;& \gamma_{i, j}' \cdot \frac{1}{1-\frac{p(n-i)}{1-p}}\\
\le\;& \gamma_{i, j}' \cdot \frac{1}{1-\frac{1}{(1-p)\log n}}.
\end{align*}
Hence, choosing $\chi := 1-\frac{1}{(1-p)\log n}$ we obtain
\[
\sum_{k=j}^n \gamma_{i, k}' \le \gamma_{i, j}' \cdot \frac{1}{1-\frac{1}{(1-p)\log n}} = \frac{\gamma_{i, j}'}{\chi}
\]
as required.
Now that all conditions are verified, we proceed by estimating the variables $u_i$.
Bounding the sum of the $\gamma_{i, j}'$-values as before,
\[
\sum_{j=i+1}^n \gamma_{i, j}' \le \sum_{j=0}^{\infty} \left(\frac{p(n-i)}{1-p}\right)^{j} \le \frac{1}{1-\frac{1}{(1-p)\log n}}.
\]
Using $1+x \le 1/(1-x)$ for $x < 1$ and \eqref{eq:bound-i(n-1)p^2} we get
\begin{align*}
u_i \le\;& p(1-p)^{n-1} \cdot (n-i) \cdot \left(1 + \frac{3}{5} \cdot \frac{i(n-i)p^2}{(1-p)^2}\right) \cdot \frac{1}{1-\frac{1}{(1-p)\log n}}\\
\le\;& p(1-p)^{n-1} \cdot (n-i) \cdot \frac{1}{\left(1 - \frac{3}{5} \cdot \frac{i(n-i)p^2}{(1-p)^2}\right) \cdot \left(1-\frac{1}{(1-p)\log n}\right)}\\
\le\;& p(1-p)^{n-1} \cdot (n-i) \cdot \frac{1}{\left(1 - \frac{3}{5} \cdot \frac{1}{(1-p)^2 \log n}\right) \cdot \left(1-\frac{1}{(1-p)\log n}\right)}\\
\le\;& p(1-p)^{n-1} \cdot (n-i) \cdot \frac{1}{1 - \frac{8}{5} \cdot \frac{1}{(1-p)^2 \log n}}.
\end{align*}
Applying Theorem~\ref{the:lower-bound-method} and recalling that the algorithm is initialized on the first fitness level with probability at least $1-\frac{1}{\log n}$ yields the lower bound
\begin{align*}
& \left(1-\frac{1}{\log n}\right)  \left(1-\frac{1}{(1-p)\log n}\right) \left(1 - \frac{8/5}{(1-p)^{2} \log n}\right) \frac{1}{p(1-p)^{n-1}} \sum_{i=\ell}^{n-1} \frac{1}{n-i}\\
\ge\;& \left(1 - \frac{18/5}{(1-p)^{2}\log n}\right) \frac{1-p}{p(1-p)^n} \sum_{i=1}^{\floor{n-\ell}} \frac{1}{i}.
\end{align*}
Since $\sum_{i=1}^{\floor{r}} 1/i \ge \ln r$ for any $r \in \R^+$,
the bound is at least
\begin{align*}
&
 \left(1 - \frac{18/5}{(1-p)^{2}\log n}\right) \frac{1-p}{p(1-p)^n} \cdot \ln\left(\min\left\{\frac{n}{\log n}, \frac{1}{p^2 n \log n}\right\}\right)\\
=\;&
 \left(1 - \frac{18/5}{(1-p)^{2}\log n}\right) \frac{1-p}{p(1-p)^n} \cdot \left(\ln\left(\min\left\{n, 1/(p^2 n)\right\}\right) - \ln(\log n)\right).
\end{align*}
Note $\ln(\log n) = \ln((\ln n)/\ln 2) = \ln \ln n - \ln \ln 2 < \ln \ln n + 0.37$.
For $p \le 1/n$ and $n \ge 91$ the lower bound simplifies to
\begin{align*}
 \left(1 - \frac{2.6}{\ln n}\right) \frac{1}{p(1-p)^n} \cdot \left(\ln n - \ln(\log n)\right)
\ge\;& \frac{\ln n - (\ln \ln n - \ln \ln 2) - 2.6}{p(1-p)^n}\\
\ge\;& \frac{\ln n - \ln \ln n - 3}{p(1-p)^n}.
\end{align*}
For $1/n \le p \le 1/(\sqrt{n}\log n)$, using again $n \ge 91$, we get
\begin{align*}
&  \left(1 - \frac{18/5}{(1-p)^{2}\log n}\right) \frac{1-p}{p(1-p)^n} \cdot \left(\ln(1/(p^2 n)) - \ln \ln n + \ln \ln 2\right)\\
 \ge\;&
\frac{\ln(1/(p^2 n)) - \ln \ln n + \ln \ln 2 -  \frac{18/5 \cdot \ln 2}{1-p} - p \cdot \ln(1/(p^2 n))}{p(1-p)^n}\\
 \ge\;&
\frac{\ln(1/(p^2 n)) - \ln \ln n + \ln \ln 2 -  \frac{18/5 \cdot \ln 2}{1-1/(\sqrt{n}\log n)} - \frac{\ln(\log^2 n)}{\sqrt{n}\log n}}{p(1-p)^n}\\
\ge\;& \frac{\ln(1/(p^2 n)) - \ln \ln n - 3}{p(1-p)^n}.
\end{align*}
\end{proof}
The above lower bound holds for a very broad class of evolutionary algorithms. This indicates what performance can be achieved by EAs using the most common mutation operator, and what the optimal mutation rate is. It is interesting to note that the lower bound does not apply to all known search heuristics, though. Some search heuristics can perform better, including local mutation operators flipping only a single bit~\cite{Droste1998}, quasirandom evolutionary algorithms~\cite{Doerr2011c}, biased mutation operators~\cite{Jansen2010}, and genetic algorithms with a fitness-invariant shuffling operator~\cite{Koetzing2011a}.

\section{A Lower Bound for all Functions with Unique Optimum}
\label{sec:unique}

Intuitively, \OM{} is the easiest function with a unique global optimum. The function gives the best possible hints to reach the optimum. This can be regarded as the task of finding a single target point in the search space. A lower bound for the time until this target is found also applies to a much broader class of functions.

We therefore consider the class of functions with a unique global optimum. This class contains all linear functions, all monotone functions (as defined in~\cite{Doerr2010c}), and all unimodal functions (when unimodality is defined as having a single local optimum). It is even much broader as it also contains many multimodal problems, needle functions, trap functions, and many more functions.

We first consider the lower bound for mutation probability~$1/n$ from Corollary~\ref{cor:lower-onemax}.
Using arguments by Doerr, Johannsen, and Winzen~\cite{Doerr2010}, we show that this lower bound transfers to all functions with a unique global optimum. This yields a more precise result than the asymptotic bound $\Omega(n \log n)$ from unbiased black-box complexity by Lehre and Witt~\cite{Lehre2010}.

In~\cite{Doerr2010} the authors proved that the expected optimization time of the \EA with mutation probability $1/n$ on \OM{} is not larger than the expected optimization time of the \EA on any other function with unique global optimum. Their proof extends to arbitrary mutation-based EAs with mutation probability~$1/n$ in a straightforward way.
\begin{theorem}
\label{the:lower-general-p=1/n}
The expected number of function evaluations for every mutation-based EA $\algo$ with mutation probability~$1/n$ on every function $f$ with $n \ge 2$ bits and a unique global optimum is at least
$en \ln n - en \ln \ln n - 3en$.
\end{theorem}
\begin{proof}
For some $a \in \{0, 1\}^n$ denote by $f_a$ the function $f(x \oplus a)$ where $\oplus$ denote the bit-wise exclusive or. Observe that this transformation does not change the behavior of a mutation-based EA in any way, \ie, all mutation-based EAs have the same runtime distribution on $f_a$ as on $f$. Hence, we do not lose generality if we transform the function $f$ in such a way that $1^n$ is the global optimum.

Let $E_\algo^f$ denote the expected optimization time of $\algo$ on $f$ and assume that the algorithm has already created search points $x_1, \dots, x_t$. Let $\Tmuea{i}$ be the minimum expected remaining optimization time for $\algo$ given that $\algo$ has only created individuals on the first $i$ fitness levels so far, formally $x_1, \dots, x_t \in A_0 \cup \dots \cup A_i$ with $A_0, \dots, A_n$ the canonical partition for \OM.

Observe that by definition, since the conditions on $x_1, \dots, x_t$ are subsequently restricted,
\[
\Tmuea{n} \le \Tmuea{n-1} \le \dots \le \Tmuea{0}.
\]
Further define a more specific and slightly modified quantity for the \EAmu: let
$\Tea{i}$ be defined like $E^{\textrm{\OM}}_{\EAmu}(i)$, but with the additional condition that the history $x_1, \dots, x_t$ contains at least one search point in $A_i$.
Since we have only added a constraint, $\Tea{i} \ge E^{\textrm{\OM}}_{\EAmu}(i)$.

Following Doerr, Johannsen, and Winzen~\cite{Doerr2010}, we now prove inductively that for all $i$ it holds $\Tmuea{i} \ge \Tea{i}$.
Clearly $\Tmuea{n} \ge {\Tea{n} = 0}$. Assume $\Tmuea{j} \ge \Tea{j}$ for all $j > i$.
Let $x'$ be the next offspring constructed by $\algo$. If the best \OM-value seen so far is at most~$i$ and $\ones{x'} = k > i$ then the expected remaining optimization time is at best $\Tmuea{k}$ (or larger). If the new offspring has a smaller number of ones, the remaining expected optimization time is still bounded below by $\Tmuea{i}$. Thus, using the assumption of our induction,
\begin{align*}
\Tmuea{i} \ge\;& 1 + \sum_{k=i+1}^n \Prob{\ones{x'} = k} \cdot \Tmuea{k} + \Prob{\ones{x'} \le i} \cdot \Tmuea{i}\\
\ge\;& 1 + \sum_{k=i+1}^n \Prob{\ones{x'} = k} \cdot \Tea{k} + \Prob{\ones{x'} \le i} \cdot \Tmuea{i}.
\end{align*}
The best distribution for $\ones{x'}$ is obtained when a parent $z$ with exactly $i$ ones is selected. A formal proof of this claim is given in~\cite[Lemma~11]{Doerr2010}. (Note that the probability of selecting such a $z$ might be 0, in which cases the real bound is even larger.)
Let $Z$ be the random number of ones when mutating $z$, then
\begin{align*}
\Tmuea{i} \ge\;& 1 + \sum_{k=i+1}^n \Prob{Z = k} \cdot \Tea{k} + \Prob{Z \le i} \cdot \Tmuea{i}.
\end{align*}
On one hand this is equivalent to
\begin{align}
\label{eq:Tmuea}
\Tmuea{i} \ge\;& \frac{1 + \sum_{k=i+1}^n \Prob{Z = k} \cdot \Tea{k}}{1 - \Prob{Z \le i}}.
\end{align}
On the other hand for the \EAmu on \OM{} we have
\begin{align*}
\Tea{i} =\;& 1 + \sum_{k=i+1}^n \Prob{Z = k} \cdot \Tea{k} + \Prob{Z \le i} \cdot \Tea{i},
\end{align*}
which is equivalent to
\begin{align}
\label{eq:Tea}
\Tea{i} =\;& \frac{1 + \sum_{k=i+1}^n \Prob{Z = k} \cdot \Tea{k}}{1 - \Prob{Z \le i}}.
\end{align}
Taking \eqref{eq:Tmuea} and \eqref{eq:Tea} together yields $\Tmuea{i} \ge \Tea{i}$. Moreover, $\Tea{i} \ge E_{\EAmu}^{\OM}(i)$.
As $\algo$ and $\EAmu$ are initialized in the same way, they share the same distribution for the initial fitness level.
We conclude $E^f_{\algo} \ge E^{\textrm{\OM}}_{\EAmu}$ and the bound follows from Corollary~\ref{cor:lower-onemax} applied to~$\EAmu$.
\end{proof}
Witt~\cite{Witt2011a} recently generalized the above proof towards arbitrary mutation probabilities and stochastic dominance. The latter is a stronger statement than a comparison of expectations. If the running time of an algorithm $\algo$ dominates that of $\mathcal{B}$ then this implies that the expected running time of $\algo$ is higher than that of~$\mathcal{B}$.

The generalization towards arbitrary mutation probabilities $p \le 1/2$ is non-trivial. In contrast to the above proof, it is not always the case that choosing the parent with the largest number of 1-bits yields the best progress. For this reason, we just cite his result here.
\begin{theorem}[Witt~\cite{Witt2011a}]
Consider a mutation-based EA $\algo$ with population size $\mu$ and mutation
probability $p \le 1/2$ on any function with a unique global optimum. Then the optimization time of $\algo$ is stochastically at least as large as the optimization time of the $\EAmu$ on \OM.
\end{theorem}

This immediately implies that the lower bound from~\ref{the:lower-onemax} transfers to every function with a unique global optimum.
\begin{theorem}
\label{the:lower-general}
The expected optimization time of every mutation-based EA using mutation probability~$p$ on every function with a unique optimum is at least
\[
\frac{\ln n - \ln \ln n - 3}{p(1-p)^n}
\]
if $2^{-n/3} \le p \le 1/n$ and at least
\[
\frac{\ln(1/(p^2n)) - \ln \ln n - 3}{p(1-p)^n}
\]
if $1/n \le p \le 1/(\sqrt{n} \log n)$.
\end{theorem}
As a side result, we have also shown that the \EAmu is an optimal algorithm for \OM{}. For every fixed value of~$\mu$ the \EAmu is never worse than any other algorithm initialized with $\mu$ uniform random individuals.
It is interesting to note that, as for \LO, the $\EA$, i.\,e., the \EAmu with $\mu=1$, is generally not the best mutation-based algorithm for \OM. In fact, for a proper choice of $\mu$ and reasonable $p$ the $\EAmu$ has a strictly smaller expected optimization time.

Compare, for instance, the \EA with the $\EAmu$ for $p=1/n$ and $\mu=\Theta(\log n)$. For both we consider the time until the algorithms find a search point with at least $n/2 + \sqrt{n}$ 1-bits. It is known that the probability that initialization creates a search point with at least $n/2 + \sqrt{n}$ 1-bits is at least a constant. Hence, with high probability the $\EAmu$ will start with at least this value after initialization. (The running time in case this does not happen is negligible.)

Contrarily, if the \EA starts with $i \le n/2 + \sqrt{n}$ 1-bits then by simple drift arguments it needs at least time $n/2+\sqrt{n}-i$ to reach a search point with fitness at least $n/2 + \sqrt{n}$. The reason is that the expected progress is clearly bounded by the expected number of flipping bits, which is~1. It is not hard to see that the \EA then needs $\Theta(\sqrt{n})$ generations in expectation to reach the threshold.

As both algorithms behave equally after having reached the threshold (modulo possible small differences for overshooting the threshold), the \EAmu is faster than the \EA by an additive term of $\Theta(\sqrt{n}) - \Theta(\log n) = \Theta(\sqrt{n})$.

Note that $\mu$ cannot be too large, either. It is known that, with high probability, the number of 1-bits in a random search point is at most $n/2 + \sqrt{n} \log n$. If $\mu = \omega(\sqrt{n} \log n)$ then the \EA gets to this threshold faster than the $\EAmu$.
\begin{theorem}
Among all mutation-based EAs the expected number of fitness evaluations on \OM{} is minimized by the \EAmu with mutation probability~$p=1/n$ and $1 < \mu = O(\sqrt{n} \log n)$.
\end{theorem}
This result contrasts the result by~Borisovsky and Eremeev~\cite{Borisovsky2008} on the optimality of the \EA on \OM. The authors do not consider the impact of initialization. Strictly speaking, their concept of dominance does not generally hold when comparing an algorithm with the \EA that is initialized in a different way.

As word of caution, we remark that it is clearly not worth optimizing for~$\mu$ in practice as the differences in the expected running time only concern additive terms of small order.

\section{An Exponential Lower Bound for Long $k$-Paths}
\label{sec:long-k-paths}

Finally, we extend the proposed lower-bound method towards settings where many transition probabilities have to be considered. A common setting is that transition probabilities to the next few higher fitness levels can be estimated quite easily. But if there are many fitness levels, dealing with those to fitness levels that are ``far away'' can become tedious. Also, in some settings condition~\eqref{eq:gamma-condition} on the transition probabilities may be violated when transition probabilities become very small. If this only happens when the transition probabilities are very small anyway, we still expect the lower bound from Theorem~\ref{the:lower-bound-method} to hold, apart from small error terms.

This reasoning is made precise in the following theorem. For each fitness level we only consider the next $d$ fitness levels, where $d \in \N$ can be chosen arbitrarily. The conditions involving transition probabilities only need to hold for these values. If $d \ll m$ this means that we only have to consider a tiny fraction of all transition probabilities.
We also introduce a variable~$\alpha$ as a lower bound for the probability that a transition is only made to these $d$ levels.
The resulting bound equals the one from Theorem~\ref{the:lower-bound-method} apart from a factor $\alpha^{m-i}$. This factor can be regarded as (an upper bound on) the probability that the algorithm on every fitness level makes jumps up to at most $d$ fitness levels.
\begin{theorem}
\label{the:lower-bound-method-capped}
Consider an algorithm~\algo and a partition of the search space into non-empty sets
$A_1, \dots, A_m$.
Choose $d \in \N$ and
let the probability of \algo{} traversing from level $i$ to level $i < j \le i+d$ in one step be at most $u_i \cdot \gamma_{i,j}$, where $\sum_{j=i+1}^{m} \gamma_{i, j} = 1$.

Define $\alpha = \alpha(d)$ such that $\alpha \le \sum_{j=1}^{d} \gamma_{i, i+j}$ for all $1 \le i \le m-d-1$.
Assume that for all $i < j \le i+d$ and some $0 \le \chi \le 1$ it holds
\begin{equation}
\label{eq:gamma-condition-capped}
\gamma_{i, j} \ge \chi \sum_{k=j}^{m} \gamma_{i, k}.
\end{equation}
Then the expected hitting time of $A_m$ is at least
\begin{align}
&
\sum_{i=1}^{m-1} \Prob{\text{$\algo$ starts in $A_i$}} \cdot \alpha^{m-i} \cdot \left(\frac{1}{u_i} + \chi \sum_{j=i+1}^{m-1} \frac{1}{u_j}\right)\label{eq:complex-lower-bound-capped}\\
\ge\;& \sum_{i=1}^{m-1} \Prob{\text{$\algo$ starts in $A_i$}} \cdot \alpha^{m-i} \chi \sum_{j=i}^{m-1} \frac{1}{u_j}.\label{eq:simple-lower-bound-capped}
\end{align}
\end{theorem}
As the proof is very similar to the proof of Theorem~\ref{the:lower-bound-method}, it is omitted.
Alternatively, the statement can be proven by conditioning on the event that in each improvement of the current best fitness level the algorithm advances by at most~$d$ levels, and applying the law of total expectation.

A prime example for a setting where the new method is applicable is the class of long $k$-paths. These functions were introduced by Horn, Goldberg, and Deb~\cite{Horn1994}, formally defined by Rudolph~\cite{Rudolph1997}, and analyzed by Droste, Jansen, and Wegener~\cite{Droste2002}. We stick to a slightly cleaner formulation from~\cite{Sudholt2009}. A long $k$-path is a sequence of search points called path. Two neighbored points on the path differ in exactly one bit. Assigning increasing fitness values to the points on the path enables an EA to climb up the path. All search points outside the path have worse fitness and they give hints to reach the start of the path.

The parameter $k$ indicates the distance between different parts of the path. For all points $x$ on the path, the $i$-th successor has Hamming distance~$i$ to~$x$, for $1 \le i \le k$. All further successors of~$x$ have Hamming distance at least $k$ to~$x$. This means that in order to take a shortcut on the path, an EA must flip at least~$k$ bits at the same time. If $k$ is not too small, an EA typically climbs to the end of the path in small steps. For $k = \sqrt{n}$ the probability of taking a shortcut is exponentially small, and the length of the path is still exponential. More precisely, the length of a long $k$-path on $n$ bits is~$k \cdot 2^{k/n}-k$~\cite{Droste2002,Sudholt2009}.

Long $k$-paths are a prime example for this extension because they give rise to a potentially exponential number of fitness values. For every point on the path, the Hamming distances to the next $k$ successors on the path are well known. But for all further search points we only know that they have Hamming distance at least~$k$. Putting $d := k$, it is easy to apply the modified lower-bound method.

For simplicity, we assume that the \EA is initialized with the first point on the long $k$-path. This not an essential restriction. It is very unlikely that the long $k$-path is reached beforehand as the ``density'' of points on the long $k$-path is extremely low, for reasonable values of~$k$. By definition, each Hamming ball of radius $k/2$ contains roughly $n^{k/2}/(k/2)!$ search points, but at most~$k$ of these can be part of the long $k$-path. This means that it is extremely unlikely to find a point on the path by chance (except for the first $k$ points), while being guided towards the start of the path.
\begin{theorem}
\label{the:lower-bound-long-k-paths}
Consider the \EA with mutation probability~$p$ starting at the first point of the long $k$-path. Let $m+1 = k \cdot 2^{n/k} - k$ be the number of search points on the long $k$-path, then the expected optimization time of the \EA is at least
\[
\frac{m}{p(1-p)^{n-1}} \cdot \left(\frac{1-2p}{1-p}\right)^2 \cdot \left(1 - \left(\frac{p}{1-p}\right)^k\right)^{m}.
\]
\end{theorem}
In order to make sense of this lower bound, note that the term $\frac{m}{p(1-p)^{n-1}}$ reflects the expected time to make $m$ specific 1-bit flips. This would be the exact expected optimization time if the \EA would never accept a mutation that flips more than one bit. It also represents an upper bound on the expected optimization time of the \EA by a straightforward application of Theorem~\ref{the:fitness-levels}.
The term $\left(\frac{1-2p}{1-p}\right)^2$ is necessary to account for successful mutations that flip more than one bit. The last term $\left(1 - \left(\frac{p}{1-p}\right)^k\right)^{m}$ roughly equals the probability that no improving mutation makes a progress by more than $k$ on the path on all fitness levels.

For the common choice $k = \sqrt{n}$ we get the following.
The bound from Theorem~\ref{the:lower-bound-long-k-paths} is simplified by applying the inequality $(1-x)^m \ge e^{-2xm}$ for $0 \le x \le 1/2$ and $m \ge 1$ to $x := (p/(1-p))^{\sqrt{n}}$.
\begin{corollary}
Consider the \EA with mutation probability $0 < p \le 1/3$ starting at the first point of the long $k$-path with $k = \sqrt{n}$. Then the expected optimization time of the \EA is at least
\[
\frac{\sqrt{n} 2^{\sqrt{n}}-\sqrt{n}}{p(1-p)^{n-1}} \cdot \left(\frac{1-2p}{1-p}\right)^2 \cdot \left(1 - 2\sqrt{n}2^{\sqrt{n}} \cdot \left(\frac{p}{1-p}\right)^{\sqrt{n}}\right).
\]
For every $0 < p = o(1)$ the expectation is
\[
\frac{\sqrt{n} 2^{\sqrt{n}}-k}{p(1-p)^{n-1}} \cdot (1-o(1)),
\]
i.\,e., upper and lower bounds are tight up to lower-order terms.
Furthermore, the choice $p=1/n$ for the mutation probability minimizes the expected number of function evaluations of the \EA in this setting if $n$ is large enough.
\end{corollary}
The dominant term for $p = 1/n$ is $en^{3/2} 2^{\sqrt{n}}$. The leading constant is by a factor of $2e$ larger than the leading constant in the previous best known lower bound $1/2 \cdot n^{3/2} 2^{\sqrt{n}}$. The latter can be derived from enhancing the proof of Theorem~23 in~\cite{Droste2002} with modern drift analysis techniques, and assuming that the \EA starts on the first point of the path.

\begin{proof}[Proof of Theorem~\ref{the:lower-bound-long-k-paths}]
Consider the canonical fitness-level partition $A_0, \dots, A_m$, i.\,e., $A_0$ contains the first point $0^n$ on the path and $A_m$ contains the last point on the path. The transition probabilities are cut off after a jump length of $d := k$, where $k$ is the parameter of the long $k$-path.
For all $0 \le i \le m$ and $1 \le j \le m-i$ define
\[
u_i = (1-p)^n \cdot \sum_{j=1}^{m-i} \left(\frac{p}{1-p}\right)^j
\]
and
\[
\gamma_{i, i+j} =
\frac{\left(\frac{p}{1-p}\right)^j}{\sum_{\ell=1}^{m-i} \left(\frac{p}{1-p}\right)^\ell}.
\]
Intuitively, by defining these values we pretend that the $j$-th successor on the path has Hamming distance~$j$, for all~$j$---not just for $1 \le j \le k$.
It is obvious from the definition that $\sum_{j=i+1}^{m} \gamma_{i, j} = 1$ for all $0 \le i \le m$. For $i < j \le i+d$ we have
\[
u_i \gamma_{i, j} = (1-p)^n \cdot \left(\frac{p}{1-p}\right)^j = p^j (1-p)^{n-j},
\]
which is precisely the probability of mutation reaching the $j$-th successor of the current search point on the path.

Define
\[
\chi = \frac{1}{\sum_{j=0}^{m-1} \left(\frac{p}{1-p}\right)^j},
\]
then for $i < j \le i+d$ condition~\ref{eq:gamma-condition-capped} resolves to
\[
\frac{\left(\frac{p}{1-p}\right)^j}{\sum_{a=1}^{m-i} \left(\frac{p}{1-p}\right)^a} \ge \frac{1}{\sum_{a=0}^{m-1} \left(\frac{p}{1-p}\right)^a} \cdot \sum_{\ell=j}^{m} \frac{\left(\frac{p}{1-p}\right)^\ell}{\sum_{a=1}^{m-i} \left(\frac{p}{1-p}\right)^a}
\]
and this is equivalent to
\[
1 \ge \frac{1}{\sum_{a=0}^{m-1} \left(\frac{p}{1-p}\right)^a} \cdot \sum_{\ell=0}^{m-j} \left(\frac{p}{1-p}\right)^\ell,
\]
which is true since $j \ge 1$.
Now for all $0 \le i \le m - d - 1$ we need to define $\alpha$ as a lower bound for
\begin{align*}
\sum_{j=1}^{d} \gamma_{i, i+j}
=\;&  \frac{\sum_{j=1}^{d}\left(\frac{p}{1-p}\right)^j}{\sum_{\ell=1}^{m-i} \left(\frac{p}{1-p}\right)^\ell}.
\end{align*}
The worst case is obtained for $i = 0$ where we get
\begin{align*}
\frac{\sum_{j=1}^{d}\left(\frac{p}{1-p}\right)^j}{\sum_{\ell=1}^{m} \left(\frac{p}{1-p}\right)^\ell}
= \frac{\frac{p}{1-p} - \left(\frac{p}{1-p}\right)^{d+1}}{\frac{p}{1-p} - \left(\frac{p}{1-p}\right)^{m+1}}
= 1 - \frac{\left(\frac{p}{1-p}\right)^{d+1} - \left(\frac{p}{1-p}\right)^{m+1}}{\frac{p}{1-p} - \left(\frac{p}{1-p}\right)^{m+1}}
\ge 1 - \left(\frac{p}{1-p}\right)^{d} := \alpha.
\end{align*}
Applying Theorem~\ref{the:lower-bound-method-capped} yields the lower bound
\begin{align*}
& \frac{1}{\sum_{j=0}^{m-1} \left(\frac{p}{1-p}\right)^j} \cdot \left(1 - \left(\frac{p}{1-p}\right)^k\right)^{m} \cdot \sum_{i=0}^{m-1} \frac{1}{(1-p)^n \cdot \sum_{j=1}^{m-i} \left(\frac{p}{1-p}\right)^j}\\
\ge\;& \frac{1-2p}{1-p} \cdot \left(1 - \left(\frac{p}{1-p}\right)^k\right)^{m} \cdot \sum_{i=0}^{m-1} \frac{1}{(1-p)^n \cdot \frac{p}{1-2p}}\\
\ge\;& \frac{m}{p(1-p)^{n-1}} \cdot \left(\frac{1-2p}{1-p}\right)^2 \cdot \left(1 - \left(\frac{p}{1-p}\right)^k\right)^{m}.
\end{align*}
\end{proof}

\section{Conclusions}

We have presented a new method for proving lower bounds on the expected optimization time of randomized search heuristics.
The method is based on an adaptation of the fitness-level method, with additional conditions on transition probabilities. It is intuitive, elegant, versatile, and easy to apply as one can freely choose values for $\chi$, $u_i$, and $\gamma_{i, j}$ ($1 \le i < j \le m$) subject to the required conditions. As a side result, it has also led to a refinement of the well-known upper-bound method with fitness levels.

The lower-bound method has been accompanied by several applications to a broad range of evolutionary algorithms. To this end, we have introduced the class of mutation-based evolutionary algorithms. It captures all EAs that only use mutation, regardless of parent selection or population models.
We have derived very precise lower bounds for \LO{}, \OM{}, and all functions with a unique global optimum. These bounds apply to all mutation-based EAs. Such a generality was previously only known for black-box complexity results.
A further application for the \EA on long $k$-paths has shown that the method still yields tight lower bounds, even when considering only a tiny fraction of all transition probabilities.

All bounds are parametrized with the mutation probability~$p$. The lower bounds for \LO, \OM{}, and long $k$-paths are tight, compared with upper bounds for the \EA, up to smaller-order terms, for all reasonable mutation probabilities.
This is a rare occasion of results that are both very general and very precise at the same time.

The results have also allowed to formally identify optimal mutation-based EAs for \LO and \OM, i.\,e., which algorithm minimizes the expected number of fitness evaluations. In both cases this is a variant of the \EA that creates more than one search point uniformly at random during initialization. Furthermore, we have seen that $p \approx 1.59/n$ is an optimal fixed mutation rate for \LO (see~\cite{Boettcher2010}), $p=1/n$ is optimal for \OM (see~\cite{Witt2011a}) and $p=1/n$ is optimal for the \EA on long $k$-paths.
These very strong conclusions further demonstrate the strength of the new lower-bound method.

Summarizing, we have made an important step forward towards understanding how EAs work, how to find optimal parameter settings, and which EAs are optimal for certain problems. Note that the method itself is not restricted to mutation-based EAs in binary spaces. It is ready to be applied to other search spaces and further stochastic search algorithms; either in its pure form or as a part of a more general analysis.

\section*{Acknowledgments}
The author was partially supported by a postdoctoral fellowship from the German Academic Exchange Service while visiting the International Computer Science Institute in Berkeley, CA, USA as well as by EPSRC grant EP/D052785/1. The author thanks Jon Rowe for suggesting the term \emph{viscosity}, Carsten Witt for insightful discussions about large mutation probabilities on \OM, and Chao Qian for pointing out an error in Theorem~\ref{the:refined-upper-bound}.

\bibliographystyle{abbrv}

\begin{thebibliography}{10}

\bibitem{Auger2011}
A.~Auger and B.~Doerr, editors.
\newblock {\em Theory of Randomized Search Heuristics -- Foundations and Recent
  Developments}.
\newblock Number~1 in Series on Theoretical Computer Science. World Scientific,
  2011.

\bibitem{Borisovsky2008}
P.~A. Borisovsky and A.~V. Eremeev.
\newblock Comparing evolutionary algorithms to the {(1+1)-EA}.
\newblock {\em Theoretical Computer Science}, 403(1):33--41, 2008.

\bibitem{Boettcher2010}
S.~B{\"o}ttcher, B.~Doerr, and F.~Neumann.
\newblock Optimal fixed and adaptive mutation rates for the leadingones
  problem.
\newblock In {\em 11th International Conference on Parallel Problem Solving
  from Nature (PPSN~2010)}, volume 6238 of {\em LNCS}, pages 1--10. Springer,
  2011.

\bibitem{Chen2010}
T.~Chen, K.~Tang, G.~Chen, and X.~Yao.
\newblock Analysis of computational time of simple estimation of distribution
  algorithms.
\newblock {\em IEEE Transactions on Evolutionary Computation}, 14(1):1--22,
  2010.

\bibitem{Cormen2001}
T.~H. Cormen, C.~E. Leiserson, R.~L. Rivest, and C.~Stein.
\newblock {\em Introduction to Algorithms}.
\newblock The MIT Press, 2nd edition, 2001.

\bibitem{Doerr2010a}
B.~Doerr, M.~Fouz, and C.~Witt.
\newblock Quasirandom evolutionary algorithms.
\newblock In {\em Genetic and Evolutionary Computation Conference (GECCO~'10)},
  pages 1457--1464. ACM Press, 2010.

\bibitem{Doerr2011c}
B.~Doerr, M.~Fouz, and C.~Witt.
\newblock Sharp bounds by probability-generating functions and variable drift.
\newblock In {\em Proceedings of the 13th Annual Genetic and Evolutionary
  Computation Conference (GECCO~'11)}, pages 2083--2090. ACM Press, 2011.

\bibitem{Doerr2010c}
B.~Doerr, T.~Jansen, D.~Sudholt, C.~Winzen, and C.~Zarges.
\newblock Optimizing monotone functions can be difficult.
\newblock In {\em 11th International Conference on Parallel Problem Solving
  from Nature (PPSN~2010)}, volume 6238 of {\em LNCS}, pages 42--51. Springer,
  2010.

\bibitem{Doerr2011a}
B.~Doerr, D.~Johannsen, T.~K\"{o}tzing, P.~K. Lehre, M.~Wagner, and C.~Winzen.
\newblock Faster black-box algorithms through higher arity operators.
\newblock In {\em Proceedings of the 11th Workshop on Foundations of Genetic
  Algorithms (FOGA~'11)}, pages 163--172. ACM Press, 2011.

\bibitem{Doerr2010}
B.~Doerr, D.~Johannsen, and C.~Winzen.
\newblock Drift analysis and linear functions revisited.
\newblock In {\em IEEE Congress on Evolutionary Computation (CEC~'10)}, pages
  1967--1974, 2010.

\bibitem{Doerr2010b}
B.~Doerr, D.~Johannsen, and C.~Winzen.
\newblock Multiplicative drift analysis.
\newblock In {\em Genetic and Evolutionary Computation Conference (GECCO~'10)},
  pages 1449--1456. ACM Press, 2010.

\bibitem{Doerr2011b}
B.~Doerr and C.~Winzen.
\newblock Towards a complexity theory of randomized search heuristics:
  Ranking-based black-box complexity.
\newblock In {\em Proceedings of 6th International Computer Science Symposium
  in Russia (CSR 2011)}, volume 6651 of {\em LNCS}, pages 15--28. Springer,
  2011.

\bibitem{Droste2006a}
S.~Droste.
\newblock A rigorous analysis of the compact genetic algorithm for linear
  functions.
\newblock {\em Natural Computing}, 5(3):257--283, 2006.

\bibitem{Droste1998}
S.~Droste, T.~Jansen, and I.~Wegener.
\newblock A rigorous complexity analysis of the (1+1) evolutionary algorithm
  for separable functions with {B}oolean inputs.
\newblock {\em Evolutionary Computation}, 6(2):185--196, 1998.

\bibitem{Droste2002}
S.~Droste, T.~Jansen, and I.~Wegener.
\newblock On the analysis of the (1+1) evolutionary algorithm.
\newblock {\em Theoretical Computer Science}, 276:51--81, 2002.

\bibitem{Droste2006}
S.~Droste, T.~Jansen, and I.~Wegener.
\newblock Upper and lower bounds for randomized search heuristics in black-box
  optimization.
\newblock {\em Theory of Computing Systems}, 39(4):525--544, 2006.

\bibitem{Gutjahr2008a}
W.~J. Gutjahr and G.~Sebastiani.
\newblock Runtime analysis of ant colony optimization with best-so-far
  reinforcement.
\newblock {\em Methodology and Computing in Applied Probability}, 10:409--433,
  2008.

\bibitem{He2004}
J.~He and X.~Yao.
\newblock A study of drift analysis for estimating computation time of
  evolutionary algorithms.
\newblock {\em Natural Computing}, 3(1):21--35, 2004.

\bibitem{Horn1994}
J.~Horn, D.~E. Goldberg, and K.~Deb.
\newblock Long path problems.
\newblock In Y.~Davidor, H.-P. Schwefel, and R.~M{\"a}nner, editors, {\em
  Parallel Problem Solving from Nature ({PPSN}\enskip {III})}, volume 866,
  pages 149--158. Springer, 1994.

\bibitem{Jagerskupper2011}
J.~J{\"a}gersk{\"u}pper.
\newblock Combining {M}arkov-chain analysis and drift analysis -- the
  (1+1)~evolutionary algorithm on linear functions reloaded.
\newblock {\em Algorithmica}, 59(3):409--424, 2011.

\bibitem{Jansen2010}
T.~Jansen and D.~Sudholt.
\newblock Analysis of an asymmetric mutation operator.
\newblock {\em Evolutionary Computation}, 18(1):1--26, 2010.

\bibitem{Jansen2002a}
T.~Jansen and I.~Wegener.
\newblock Evolutionary algorithms -- how to cope with plateaus of constant
  fitness and when to reject strings of the same fitness.
\newblock {\em IEEE Transactions on Evolutionary Computation}, 5(6):589--599,
  2002.

\bibitem{Jansen2011}
T.~Jansen and C.~Zarges.
\newblock Analysis of evolutionary algorithms: from computational complexity
  analysis to algorithm engineering.
\newblock In {\em Proceedings of the 11th Workshop on Foundations of Genetic
  Algorithms (FOGA~'11)}, pages 1--14. ACM, 2011.

\bibitem{Koetzing2011a}
T.~K{\"o}tzing, D.~Sudholt, and M.~Theile.
\newblock How crossover helps in pseudo-{B}oolean optimization.
\newblock In {\em Proceedings of the 13th Annual Genetic and Evolutionary
  Computation Conference (GECCO~2011)}, pages 989--996. ACM Press, 2011.

\bibitem{Lassig2010a}
J.~L{\"a}ssig and D.~Sudholt.
\newblock General scheme for analyzing running times of parallel evolutionary
  algorithms.
\newblock In {\em 11th International Conference on Parallel Problem Solving
  from Nature (PPSN~2010)}, volume 6238 of {\em LNCS}, pages 234--243.
  Springer, 2010.

\bibitem{Lassig2011}
J.~L{\"a}ssig and D.~Sudholt.
\newblock Adaptive population models for offspring populations and parallel
  evolutionary algorithms.
\newblock In {\em Proceedings of the 11th Workshop on Foundations of Genetic
  Algorithms (FOGA~2011)}, pages 181--192. ACM Press, 2011.

\bibitem{Lehre2011}
P.~K. Lehre.
\newblock Fitness-levels for non-elitist populations.
\newblock In {\em Proceedings of the 13th Annual Genetic and Evolutionary
  Computation Conference (GECCO~'11)}, pages 2075--2082. ACM Press, 2011.

\bibitem{Lehre2010a}
P.~K. Lehre.
\newblock Negative drift in populations.
\newblock In {\em 11th International Conference on Parallel Problem Solving
  from Nature (PPSN~2010)}, volume 6238 of {\em LNCS}, pages 244--253.
  Springer, 2011.

\bibitem{Lehre2010}
P.~K. Lehre and C.~Witt.
\newblock Black box search by unbiased variation.
\newblock In {\em Genetic and Evolutionary Computation Conference (GECCO~'10)},
  pages 1441--1448, 2010.

\bibitem{Neumann2009}
F.~Neumann, D.~Sudholt, and C.~Witt.
\newblock Analysis of different {MMAS} {ACO} algorithms on unimodal functions
  and plateaus.
\newblock {\em Swarm Intelligence}, 3(1):35--68, 2009.

\bibitem{Neumann2010a}
F.~Neumann, D.~Sudholt, and C.~Witt.
\newblock A few ants are enough: {ACO} with iteration-best update.
\newblock In {\em Genetic and Evolutionary Computation Conference (GECCO~'10)},
  pages 63--70, 2010.

\bibitem{Neumann2009b}
F.~Neumann and C.~Witt.
\newblock Runtime analysis of a simple ant colony optimization algorithm.
\newblock {\em Algorithmica}, 54(2):243--255, 2009.

\bibitem{BookNeuWit}
F.~Neumann and C.~Witt.
\newblock {\em Bioinspired Computation in Combinatorial Optimization --
  Algorithms and Their Computational Complexity}.
\newblock Springer, 2010.

\bibitem{Oliveto2007}
P.~S. Oliveto, J.~He, and X.~Yao.
\newblock Time complexity of evolutionary algorithms for combinatorial
  optimization: A decade of results.
\newblock {\em International Journal of Automation and Computing},
  4(3):281--293, 2007.

\bibitem{Oliveto2011}
P.~S. Oliveto and C.~Witt.
\newblock Simplified drift analysis for proving lower bounds in evolutionary
  computation.
\newblock {\em Algorithmica}, 59(3):369--386, 2011.

\bibitem{RowePersonal}
J.~Rowe.
\newblock Personal communication, 2011.

\bibitem{Rowe2011}
J.~E. Rowe and M.~D. Vose.
\newblock Unbiased black box search algorithms.
\newblock In {\em Proceedings of the 13th Annual Genetic and Evolutionary
  Computation Conference (GECCO~'11)}, pages 2035--2042. ACM Press, 2011.

\bibitem{Rudolph1997a}
G.~Rudolph.
\newblock {\em Convergence Properties of Evolutionary Algorithms}.
\newblock Verlag Dr. Kova\v{c}, 1997.

\bibitem{Rudolph1997}
G.~Rudolph.
\newblock How mutation and selection solve long-path problems in polynomial
  expected time.
\newblock {\em Evolutionary Computation}, 4(2):195--205, 1997.

\bibitem{Sudholt2009}
D.~Sudholt.
\newblock The impact of parametrization in memetic evolutionary algorithms.
\newblock {\em Theoretical Computer Science}, 410(26):2511--2528, 2009.

\bibitem{Sudholt2010a}
D.~Sudholt.
\newblock General lower bounds for the running time of evolutionary algorithms.
\newblock In {\em 11th International Conference on Parallel Problem Solving
  from Nature (PPSN~2010)}, volume 6238 of {\em LNCS}, pages 124--133.
  Springer, 2010.

\bibitem{Sudholt2010}
D.~Sudholt.
\newblock Hybridizing evolutionary algorithms with variable-depth search to
  overcome local optima.
\newblock {\em Algorithmica}, 59(3):343--368, 2011.

\bibitem{Sudholt2011a}
D.~Sudholt and C.~Thyssen.
\newblock Running time analysis of ant colony optimization for shortest path
  problems.
\newblock {\em Journal of Discrete Algorithms}, 2011.
\newblock To appear.

\bibitem{Sudholt2008c}
D.~Sudholt and C.~Witt.
\newblock Runtime analysis of {Binary PSO}.
\newblock In {\em Proceedings of the Genetic and Evolutionary Computation
  Conference (GECCO~'08)}, pages 135--142. ACM Press, 2008.

\bibitem{Sudholtsubmitteda}
D.~Sudholt and C.~Witt.
\newblock Runtime analysis of a binary particle swarm optimizer.
\newblock {\em Theoretical Computer Science}, 411(21):2084--2100, 2010.

\bibitem{Sudholt2010b}
D.~Sudholt and C.~Zarges.
\newblock Analysis of an iterated local search algorithm for vertex coloring.
\newblock In {\em 21st International Symposium on Algorithms and Computation
  (ISAAC 2010)}, volume 6506 of {\em LNCS}, pages 340--352. Springer, 2010.

\bibitem{Wegener2002}
I.~Wegener.
\newblock Methods for the analysis of evolutionary algorithms on
  pseudo-{Boolean} functions.
\newblock In R.~Sarker, X.~Yao, and M.~Mohammadian, editors, {\em Evolutionary
  Optimization}, pages 349\protect\nobreakdash--369. Kluwer, 2002.

\bibitem{Wegener2005c}
I.~Wegener and C.~Witt.
\newblock On the optimization of monotone polynomials by simple randomized
  search heuristics.
\newblock {\em Combinatorics, Probability and Computing}, 14:225--247, 2005.

\bibitem{Witt2006}
C.~Witt.
\newblock Runtime analysis of the {($\mu$+1) EA} on simple pseudo-{Boolean}
  functions.
\newblock {\em Evolutionary Computation}, 14(1):65--86, 2006.

\bibitem{Witt2009}
C.~Witt.
\newblock Why standard particle swarm optimisers elude a theoretical runtime
  analysis.
\newblock In {\em Foundations of Genetic Algorithms 10 (FOGA~'09)}, pages
  13--20. ACM Press, 2009.

\bibitem{Witt2011a}
C.~Witt.
\newblock Tight bounds on the optimization time of the {(1+1)~EA} on linear
  functions.
\newblock {\em ArXiv e-prints}, Aug. 2011.
\newblock Available from \url{http://arxiv.org/abs/1108.4386v1}.

\bibitem{Zarges2008}
C.~Zarges.
\newblock Rigorous runtime analysis of inversely fitness proportional mutation
  rates.
\newblock In {\em Parallel Problem Solving from Nature - PPSN X}, volume 5199
  of {\em LNCS}, pages 112--122. Springer, 2008.

\end{thebibliography}

\end{document}